\documentclass[11pt]{article}

\usepackage[final]{acl}
\usepackage{placeins}

\usepackage{times}
\usepackage{latexsym}
\usepackage{subcaption} 

\usepackage[most]{tcolorbox}
\usepackage{listings}

\usepackage[T1]{fontenc}

\usepackage[utf8]{inputenc}
\usepackage{multirow}
\usepackage{array}
\usepackage[most]{tcolorbox}

\usepackage{microtype}
\usepackage{amsmath}
\usepackage{booktabs} 
\usepackage{inconsolata}

\usepackage{amssymb}   

\usepackage{graphicx}

\usepackage{amsthm}

\theoremstyle{definition}
\newtheorem{definition}{Definition}[subsection]

\theoremstyle{plain}
\newtheorem{theorem}[definition]{Theorem}
\newtheorem{lemma}[definition]{Lemma}

\theoremstyle{definition}
\newtheorem{define}{Definition}[subsection]
\newtheorem{prop}[definition]{Proposition}

\usepackage{graphicx}
\usepackage{subcaption}

\usepackage{booktabs}   
\usepackage{multirow}   
\usepackage{makecell}   

\usepackage{titletoc} 
\usepackage{tabularx} 

\newtcolorbox{promptcard}[1]{
    enhanced,
    breakable,
    colback=white,
    colframe=black!35,
    colbacktitle=black!55,
    coltitle=white,
    title={#1},
    boxrule=0.5pt,
    arc=2pt,
    left=6pt,
    right=6pt,
    top=5pt,
    bottom=5pt,
    fontupper=\footnotesize,
    width=\linewidth
}

\title{CompEvo: Competition-Induced Evolution for Multi-Agent in News-Driven Time Series Forecasting}

\author{
  \textbf{Yuxuan Zhang}\textsuperscript{1,*} \quad
  \textbf{Yangyang Feng}\textsuperscript{1,*} \quad
  \textbf{Yong Guan}\textsuperscript{2} \quad
  \textbf{Daifeng Li}\textsuperscript{1,\textdagger} \\
  \textbf{Kexin Zhang}\textsuperscript{1} \quad
  \textbf{Junlan Chen}\textsuperscript{1} \quad
  \textbf{Bowen Deng}\textsuperscript{1} \quad
  \textbf{Jun Liu}\textsuperscript{3} \quad
  \textbf{Zehua Zeng}\textsuperscript{3} \\
  \textsuperscript{1}Sun Yat-sen University, Guangzhou, China \\
  \textsuperscript{2}North China Electric Power University, Beijing, China \\
  \textsuperscript{3}Unilumin Group Co., Ltd., Shenzhen, China \\
  \texttt{\{zhangyx528,fengyy33\}@mail2.sysu.edu.cn} \quad
  \texttt{lidaifeng@mail.sysu.edu.cn} \\
  \textsuperscript{*}Equal contribution. \quad
  \textsuperscript{\textdagger}Corresponding author.
}

\newif\ifwithappendix
\withappendixtrue

\begin{document}
\maketitle

\begin{abstract}


News-driven time series forecasting (NTSF) uses evolving textual events together with historical observations to predict future values, supporting applications such as market risk monitoring and resource scheduling. In multi-agent settings, two challenges still remain. The first is degeneration of thought, where agents converge to similar evidence-seeking behaviors. The second is insufficient theoretical grounding, where strategy updates are often heuristic and lack a principled formulation. To address the above challenges, we propose \textbf{CompEvo}, a competition-induced evolution framework for multi-agent NTSF. At the theoretical level, our evolutionary game formulation establishes equilibrium existence and local convergence under the stated assumptions. At the methodological level, we construct a trainable multi-agent evolution framework that integrates strategy execution, fitness-based differentiable selection, and competition-induced strategy evolution. CompEvo enables heterogeneous agents to explore diverse news evidence, converts forecasting feedback into differentiable influence weights, and evolves agent strategies under competitive pressure to preserve effective logic while maintaining diversity. Experiments on four real-world datasets show that CompEvo reduces RMSE by 29.8\% and MAPE by 30.3\% on average relative to the strongest multi-agent baseline. Further analysis indicates that CompEvo maintains diverse and specialized agent behaviors\footnote{The code and data are available at \url{https://github.com/Aliiina-z/CompEvo}.}.





\end{abstract}

\section{Introduction}

\begin{figure}[t]
  \centering
  \includegraphics[width=\linewidth]{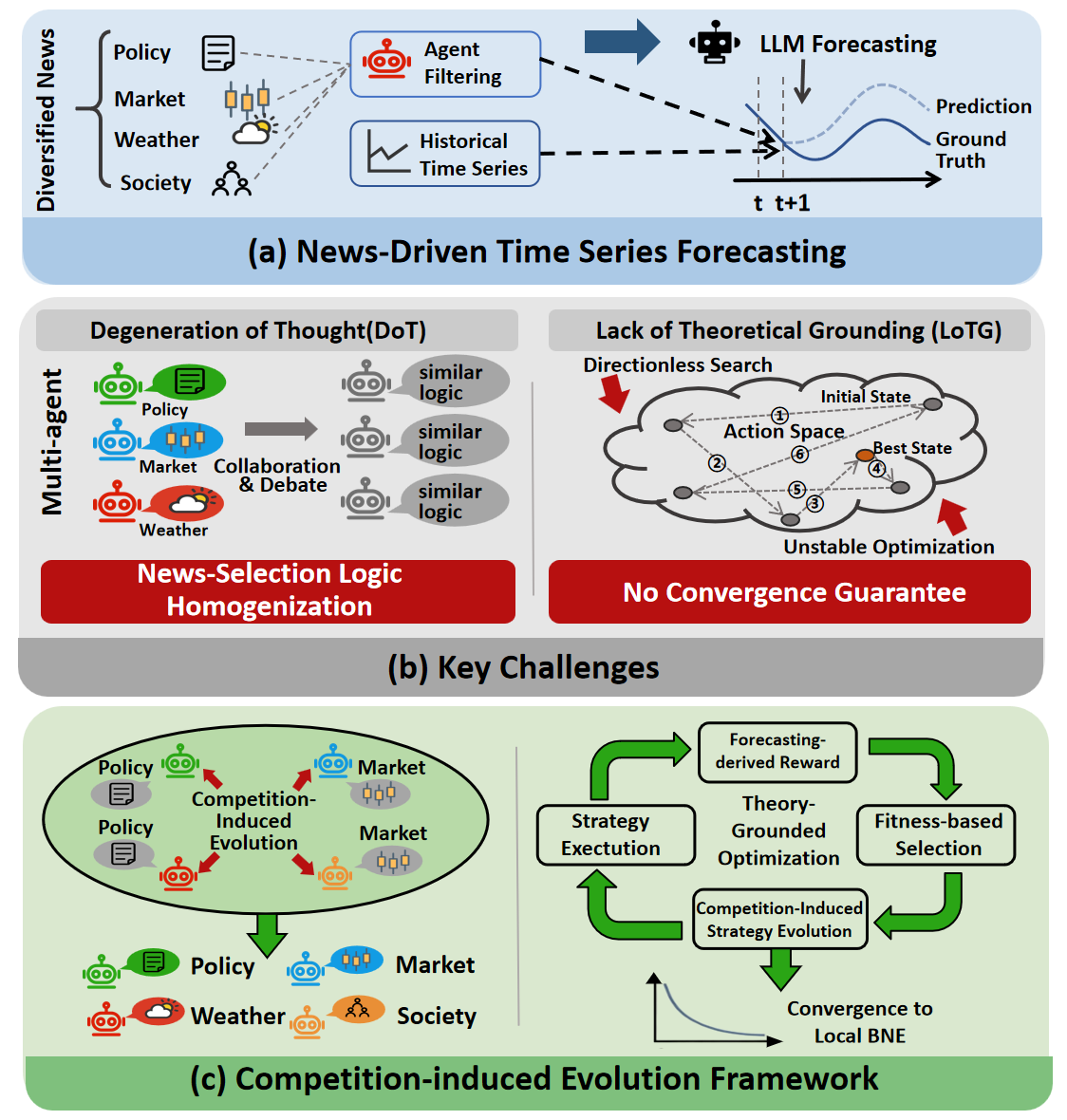}
\caption{
Illustration of CompEvo.
}
  \label{fig:intro_framework}
  \vspace{-15pt}
\end{figure}




News-driven time series forecasting (NTSF) predicts future numerical values by combining historical observations with time-aligned news events \cite{Wang2024, xiong2025beyond}. This task is vital in settings such as electricity demand planning, financial risk monitoring, and traffic management, where external events can abruptly alter future trends. Accurately capturing these evolving external factors is therefore essential. In the absence of explicit prior knowledge, the model must explore evidence-seeking strategies to identify informative news signals for forecasting.

Recent LLM-based agent methods improve NTSF by retrieving and reasoning over news evidence. However, a single agent often relies on a fixed evidence-seeking strategy, potentially missing complementary evidence. Multi-agent discussion and debate frameworks provide a cooperative paradigm to introduce diverse perspectives \cite{wang2026maspo, nguyen2026hearboth}, but applying them to NTSF encounters two challenges. The first is \textbf{degeneration of thought (DoT)}, where agents gradually imitate peer reasoning and converge to homogeneous evidence-seeking strategies \cite{chan2024chateval, Estornell2024MultiLLMDebate}. This homogenization causes overlapping evidence retrieval and restricted exploration. The second is \textbf{insufficient theoretical grounding (ITG)}, where existing competitive or evolutionary multi-agent designs rely on heuristic prompt-level interactions. Despite the emergence of end-to-end debate frameworks based on reinforcement learning (RL), they lack formal proofs of convergence and stability. Without a principled objective linking forecasting feedback to strategy updates, these methods lack theoretical guarantees for stable optimization and strategy retention.


Interestingly, evolutionary game theory (EGT), a paradigm widely used to model innovation and strategy adaptation under partial observations \cite{TomIScience2024, PatraPhy2025}, provides an insightful perspective for this task. Under EGT, individuals receive payoffs from their actions, successful strategies expand via selection, and continuous variation prevents population homogenization \cite{qu2026coral}. This process mirrors NTSF, where agents update their evidence-seeking strategies according to newly observed news evidence and forecasting feedback.

At the methodological level, we propose \textbf{CompEvo}, a competition-induced evolution framework that models agents as a competing population (Figure~\ref{fig:intro_framework}). CompEvo consists of three modules: \textit{strategy execution}, \textit{fitness-based selection}, and \textit{competition-induced strategy evolution}. First, strategy execution prompts agents to seek news evidence under their current strategies. Second, fitness-based selection converts forecasting feedback into differentiable influence weights and transforms strategy retention into an optimizable process, ensuring that strategy updates remain theoretically grounded. Third, competition-induced strategy evolution updates agent strategies with partial peer information, encouraging diverse evidence exploration while avoiding homogenization.

At the theoretical level, we formulate the multi-agent interaction as an incomplete-information evolutionary game and establish the existence of a mixed-strategy Bayesian Nash equilibrium. We further analyze local convergence, long-term regret, and fitness-weight monotonicity under the stated assumptions, providing a formal account of stable strategy retention under differentiable optimization. Our contributions are summarized as follows:


\vspace{-0.2cm}
\begin{itemize}
\item We formulate multi-agent NTSF as an incomplete-information evolutionary game and characterize equilibrium existence, local convergence, long-term regret, and fitness-weight monotonicity under the stated assumptions.
\vspace{-0.3cm}
\item We propose CompEvo, a trainable framework that integrates strategy execution, fitness-based selection, and competition-induced strategy evolution, enabling agents to continuously evolve and maintain diverse, specialized evidence-seeking behaviors.
\vspace{-0.3cm}
\item We evaluate CompEvo on four real-world benchmarks, reducing RMSE by 29.8\% and MAPE by 30.3\% relative to the strongest multi-agent baseline, and release the implementation and data to support reproducibility.

\end{itemize}

\section{Related Work}

\subsection{LLMs for News-Driven Forecasting}

Current research on LLM based time series forecasting can be grouped into three categories. Pure numerical forecasting maps historical observations into token space via reprogramming or fine-tuning to capture temporal dependencies, largely bypassing external context \cite{zhou2024one, guo2026tllm, qiu2026rethinking}. Text-aware forecasting incorporates aligned news signals or temporal descriptions generated by LLMs to model event-driven fluctuations \cite{liu2025dpgpt4mts, xiong2025beyond}. Agentic and multi-agent forecasting further enables LLM agents to retrieve, filter, refine, or discuss news evidence before prediction \cite{Wang2024, chen-etal-2025-semob, park-etal-2025-maporl}. However, these methods typically rely on fixed evidence-seeking strategies or consensus based coordination, making them less suitable for NTSF, where heterogeneous news selection strategies can remain valid simultaneously and should evolve under forecasting feedback.

\subsection{Evolutionary Game-Theoretic Approaches}
\label{sec:related_work_mas_egt}

EGT provides mathematical foundations for long-term strategy adaptation and population-level selection \cite{smith1973logic}. Recent studies introduce EGT principles into LLM systems for agent generation, prompt-level selection, and discrete strategy mutation \cite{yuan-etal-2025-evoagent, qu2026coral}, while competitive network topologies offer alternative interaction structures for LLM agents \cite{XuWolf2025}. However, existing frameworks rely on discrete or prompt-level operations, lacking a differentiable mechanism that preserves diverse evidence-seeking strategies and links forecasting feedback to strategy retention. This gap motivates our formulation of news-driven forecasting as a competition-induced evolutionary process with differentiable selection and strategy evolution.


\begin{figure*}[htbp]
  \includegraphics[width=\linewidth]{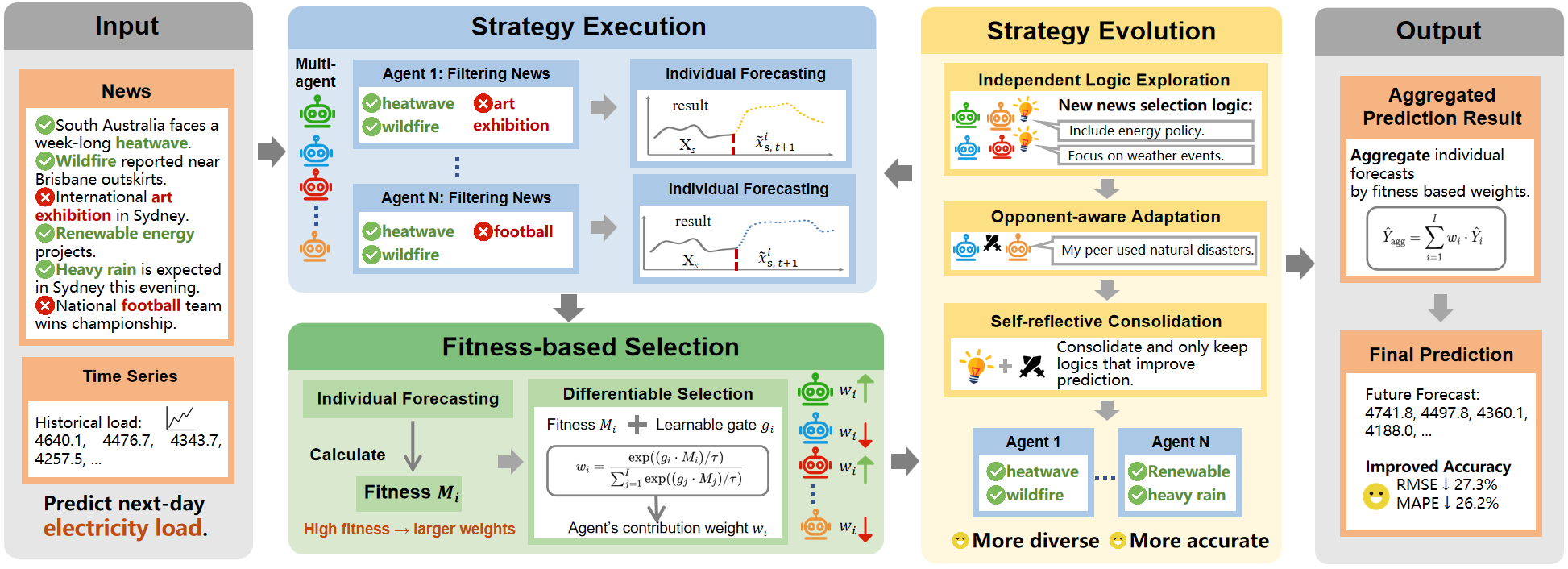}
  \centering
  \caption{Overview of CompEvo for multi-agent NTSF.}
  \vspace{-15pt}
  \label{framework}
\end{figure*}

\section{Preliminaries}

We formulate NTSF as a multi-agent competitive process over textual strategy representations. Let $\mathbf{x}_{t-L:t-1} = \{x_{t-L}, \dots, x_{t-1}\}$ denote the historical numerical observations of length $L$, and $\mathcal{C}_t$ represent the time-aligned public news pool available at forecasting step $t$. The core objective is to predict the future numerical value $x_t \in \mathbb{R}$ via a unified mapping.

\vspace{-0.3cm}
\begin{equation}
\hat{x}_t = \mathcal{F}(\mathbf{x}_{t-L:t-1}, \mathcal{C}_t).
\label{eq:prelim_task}
\end{equation}
\vspace{-0.6cm}

A population of $I$ heterogeneous agents competes to optimize this forecasting task. Each agent $i \in \{1, \dots, I\}$ follows an agent-specific evidence-seeking strategy to filter and interpret the public news pool. Under this strategy, agent $i$ extracts a customized news subset $\mathcal{C}_{t,i} \subseteq \mathcal{C}_t$ and generates an individual prediction $\hat{x}_{t,i} = \mathcal{F}_i(\mathbf{x}_{t-L:t-1}, \mathcal{C}_{t,i})$. The agents observe identical historical data but differ in how their strategies transform shared raw evidence into distinct forecasting behaviors.


\section{Methodology}
\label{sec:methodology}




As illustrated in Figure~\ref{framework}, CompEvo is a trainable multi-agent framework built upon a shared LLM backbone, where agents maintain separate parameter-efficient fine-tuning (PEFT) parameters to preserve agent-specific adaptation. Each agent $i$ maintains a parameterized strategy $\pi_i(a \mid o_i, \theta_i)$, where $\theta_i$ denotes the agent-specific training parameters and $a$ denotes the action of the generated strategy. In our implementation, this action is performed as textual logic $l_i$, which specifies how the agent retrieves the news evidence and produces an individual forecast. We therefore use strategy to refer to the parameterized policy-level object and logic to refer to its generated textual realization. Agents share the same backbone for efficiency, while agent-specific adapters and competitive updates produce heterogeneous textual logic realizations.

\subsection{Strategy Execution}


Strategy execution specifies how each agent samples or generates textual logic from its current strategy, then uses this logic to filter relevant external evidence and produce an individual forecast. To model competition under incomplete information, the observable state of agent $i$ in round $e \leq E$ is defined as follows\footnote{A round denotes one strategy-update cycle among agents over a data batch, while an epoch consists of multiple rounds and denotes one full pass over the training data.}.

\vspace{-0.3cm}
\begin{equation}
o_i^{(e)} = \{X_s, \mathcal{C}_s, l_i^{(e-1)}, l_{-i}^{(e-1)}\},
\label{eq:agent_state}
\end{equation}
where $X_s$ denotes the historical time series context, $\mathcal{C}_s$ denotes the shared candidate news pool, $l_i^{(e-1)}$ is the logic used by agent $i$ in the previous round, and $l_{-i}^{(e-1)}$ denotes the exposed logic of other agents. The agent samples a strategy action from the policy $\pi_i^{(e)}(a \mid o_i^{(e)}, \theta_i)$. This action is realized as the updated textual logic $l_i^{(e)}$. For optimization, the textual logic is further represented by its hidden state $h_i^{(e)}$. Implementation details are deferred to Appendix~\ref{sec:implementation_details}.

Conditioned on $l_i^{(e)}$, agent $i$ uses its logic as an agent-specific retrieval-and-interpretation rule to obtain a news subset $\mathcal{C}_{i,s}^{(e)}$ from the shared candidate pool $\mathcal{C}_s$ and produce an individual 48-step prediction vector $\hat{\mathbf{Y}}_{i,s}^{(e)} \in \mathbb{R}^{48}$ for the target window $\mathbf{Y}_s \in \mathbb{R}^{48}$. Keeping these forecasts separate converts their errors into agent-level rewards for selection and evolution instead of erasing specialization through early averaging.

\subsection{Fitness-based Selection}

Fitness-based selection converts forecasting feedback into differentiable agent weights for group prediction.
To reduce the effect of noisy results from a single round, we distinguish the immediate reward of an agent from its long-term fitness across rounds, and use the latter to guide smooth selection among agents. The reward of agent $i$ in round $e$ is defined as the negative forecasting error.
\vspace{-0.2cm}
\begin{equation}
R_i^{(e)} = -\mathcal{L}_{\text{metric}}(\hat{\mathbf{Y}}_{i,s}^{(e)}, \mathbf{Y}_s),
\label{eq:reward}
\end{equation}
where $\mathcal{L}_{\text{metric}}$ is a differentiable forecasting metric and $\mathbf{Y}_s$ is the ground-truth 48-step target vector. To capture stable competitiveness beyond one-round fluctuations, we maintain an exponential moving average of historical rewards.
\vspace{-0.2cm}
\begin{equation}
M_i^{(e+1)} = \beta M_i^{(e)} + (1-\beta) R_i^{(e)},
\label{eq:fitness}
\end{equation}
where $\beta \in [0,1)$ controls reward smoothing. Selection is then implemented through a differentiable weighting rule.
\begin{equation}
\label{eq:weight_softmax}
w_i = \frac{\exp((g_i \cdot M_i)/\tau)}{\sum_{j=1}^{I} \exp((g_j \cdot M_j)/\tau)},
\end{equation}
where $g_i$ is a learnable gate and $\tau$ is the temperature. Agents with higher fitness gain larger influence, and the learned weights define aggregation and forecasting loss.
\vspace{-0.3cm}
\begin{equation}
\hat{\mathbf{Y}}_{\text{agg}} = \sum_{i=1}^{I} w_i \hat{\mathbf{Y}}_{i,s}^{(e)},
\end{equation}
where $\hat{\mathbf{Y}}_{\text{agg}}$ is the aggregated prediction and $\hat{\mathbf{Y}}_{i,s}^{(e)}$ is the individual prediction vector of agent $i$. The forecasting objective is defined below.
\vspace{-0.3cm}
\begin{equation}
L_{\text{pred}} = \mathcal{L}_{\text{metric}}(\hat{\mathbf{Y}}_{\text{agg}}, \mathbf{Y}_s).
\label{eq:predict_loss}
\end{equation}


\subsection{Competition-Induced Strategy Evolution}

Strategy evolution updates each agent's strategy action, which is realized as textual logic, after observing its own state and partial peer information.

\noindent \textbf{Opponent context.}
An agent cannot know whether a new candidate logic will perform well before using it for forecasting. Each agent therefore updates its next-round textual logic from both its own local state and the exposed textual logic of other agents. This lets the agent use peer information without fully replacing its own reasoning direction.

\noindent \textbf{Three-stage update.}
We implement the logic update with a three-stage gated generator. First, agent $i$ generates a candidate logic representation from its current state.
\begin{equation}
\bar{h}_i^{(e)} = \mathrm{LLM}_{\mathrm{LE}}(o_i^{(e)}),
\end{equation}
where $o_i^{(e)}$ is the current state of agent $i$, $\bar{h}_i^{(e)} \in \mathbb{R}^{d}$ denotes the candidate logic representation, and $\mathrm{LLM}_{\mathrm{LE}}$ denotes the logic-evolution module. This stage preserves independent exploration across agents. Second, the agent forms an opponent context from the exposed logic representations of other agents in the previous round.
\vspace{-0.1cm}
\begin{equation}
c_i^{(e)} = \sum_{j \neq i} \alpha_{ij}\bar{h}_j^{(e-1)},
\end{equation}
where $c_i^{(e)}$ denotes the opponent context, $\alpha_{ij}$ denotes the interaction weight from agent $j$ to agent $i$, and $\bar{h}_j^{(e-1)}$ is the candidate logic representation of agent $j$ from the previous round. $\sigma(\cdot)$ is the sigmoid function and $\odot$ denotes element-wise multiplication. The opponent context is then fused with the agent's previous logic through a gate.
\begin{gather}
g_{2,i} = \sigma\!\left(W_{g2}[\bar{h}_i^{(e-1)}, c_i^{(e)}] + b_{g2}\right), \\
h_i^{m,(e)} = (1-g_{2,i}) \odot \bar{h}_i^{(e-1)} + g_{2,i} \odot c_i^{(e)}.
\end{gather}
where $g_{2,i}$ is the second-stage gate, $W_{g2}$ and $b_{g2}$ are trainable projection parameters, and $h_i^{m,(e)}$ denotes the opponent-informed logic representation. Third, the agent combines the newly generated candidate with the opponent-informed representation.
\vspace{-0.2cm}
\begin{gather}
g_{3,i} = \sigma\!\left(W_{g3}[h_i^{m,(e)}, \bar{h}_i^{(e)}] + b_{g3}\right), \\
h_i^{(e)} = (1-g_{3,i}) \odot h_i^{m,(e)} + g_{3,i} \odot \bar{h}_i^{(e)}.
\end{gather}
where $g_{3,i}$ is the third-stage gate, $W_{g3}$ and $b_{g3}$ are trainable projection parameters, and $h_i^{(e)}$ denotes the final logic representation. The final representation is decoded into the next textual logic and used in the following forecasting round.

\noindent \textbf{Textual logic interface.}
We instantiate this update in the backbone's representation space. Each candidate logic is tokenized by the shared Llama-3.1-8B tokenizer and encoded with the corresponding agent's logic adapter; the last-layer state at the final non-padding token gives a 4,096-dimensional representation. Row-normalized interaction scores $\alpha_{ij}$ form the opponent context, and the two gates fuse it with the agent representation elementwise. The fused state is projected to one soft-prompt embedding and decoded autoregressively into the textual logic used in the next forecasting round. Exact adapter and optimization settings remain in Appendix~\ref{sec:implementation_details}.

\noindent \textbf{Policy optimization.}
We optimize logic generation with GRPO~\cite{deepseek_math_grpo}. GRPO uses relative rewards within the agent group as the following policy improvement signal.
\begin{equation}
\begin{aligned}
L_{\text{pg}}(\theta_i)
&= \mathbb{E}_{a \sim \pi_{i,\theta_i^{\mathrm{old}}}(a|o_i)} \Bigg[
\frac{\pi_{i,\theta_i}(a|o_i)}{\pi_{i,\theta_i^{\mathrm{old}}}(a|o_i)} \\
&\qquad\qquad {}\times \hat{A}_i(a, o_i) \Bigg].
\end{aligned}
\label{eq:pg_loss}
\end{equation}
where $a$ denotes a generated token in the textual logic $l_i$, $\pi_{i,\theta_i^{\mathrm{old}}}$ is the previous policy of agent $i$, and $\hat{A}_i(a,o_i)$ is its group-relative advantage. Further details are deferred to Appendix~\ref{sec:implementation_details}. 

\begin{table*}[htbp]
\centering
\setlength{\tabcolsep}{5pt}
\scriptsize
\renewcommand{\arraystretch}{0.5}
\begin{tabular}{llcccccccccc}
\toprule
\textbf{Dataset} & \textbf{Metrics} & \textbf{CompEvo} & \textbf{DLin.} & \textbf{iTrans.} & \textbf{FiLM} & \textbf{Pyra.} & \textbf{FED.} & \textbf{GPT4TS} & \textbf{N2F} & \textbf{GPT4MTS} & \textbf{LTP-LLM} \\
\midrule

\multirow{4}{*}{\textbf{Electricity}}
& MAE & \textbf{220.15 $\pm$ 2.4} & 255.70 & 233.58 & 254.05 & 544.64 & 238.77 & 236.91 & 250.71 & \underline{229.40} & 232.10 \\
& MSE$_{\times10^{3}}$ & \underline{125.44 $\pm$ 1.8} & 135.05 & \textbf{97.61} & 153.90 & 625.00 & 133.96 & 142.60 & 192.24 & 133.50 & 136.80 \\
& RMSE & \underline{354.18 $\pm$ 2.5} & 367.49 & \textbf{312.42} & 392.38 & 790.54 & 365.44 & 377.62 & 438.45 & 365.38 & 369.86 \\
& MAPE (\%) & \textbf{6.45 $\pm$ 0.05} & 7.29 & 6.86 & 6.81 & 8.26 & 7.27 & \underline{6.72} & 7.60 & \underline{6.72} & 6.85 \\

\midrule

\multirow{4}{*}{\textbf{Exchange}}
& MAE$_{\times10^{-3}}$ & \textbf{4.35 $\pm$ 0.08} & 6.96 & 27.04 & 5.24 & 40.18 & 35.19 & 15.05 & 13.43 & 4.82 & \underline{4.75} \\
& MSE$_{\times10^{-4}}$ & \textbf{0.35 $\pm$ 0.01} & 9.06 & 11.59 & 0.77 & 24.50 & 18.45 & 4.01 & 17.72 & \underline{0.46} & 0.49 \\
& RMSE$_{\times10^{-2}}$ & \textbf{0.59 $\pm$ 0.01} & 9.52 & 3.41 & 0.88 & 4.95 & 4.30 & 2.00 & 4.21 & \underline{0.68} & 0.70 \\
& MAPE (\%) & \textbf{0.60 $\pm$ 0.02} & 0.92 & 3.96 & \underline{0.70} & 5.93 & 5.17 & 1.34 & 5.69 & 0.75 & 0.78 \\

\midrule

\multirow{4}{*}{\textbf{Traffic}}
& MAE$_{\times10^{-2}}$ & \textbf{1.48 $\pm$ 0.03} & 1.70 & \underline{1.56} & 1.61 & 1.69 & 1.74 & 1.64 & 1.84 & 1.58 & \underline{1.56} \\
& MSE$_{\times10^{-3}}$ & \textbf{0.95 $\pm$ 0.02} & 1.67 & 1.54 & 1.49 & \underline{0.97} & 1.43 & 1.45 & 1.81 & 1.25 & 1.18 \\
& RMSE$_{\times10^{-2}}$ & \textbf{3.08 $\pm$ 0.04} & 4.09 & 3.93 & 3.86 & \underline{3.12} & 3.79 & 3.81 & 4.26 & 3.54 & 3.44 \\
& MAPE (\%) & \textbf{3.75 $\pm$ 0.10} & 4.55 & 4.12 & 4.28 & \underline{3.85} & 4.68 & 4.33 & 4.20 & 4.15 & 4.08 \\

\midrule

\multirow{4}{*}{\textbf{Bitcoin}}
& MAE$_{\times10^{-3}}$ & \textbf{2.65 $\pm$ 0.05} & 5.74 & 3.20 & 3.17 & 9.22 & 3.96 & 2.84 & 5.10 & 2.92 & \underline{2.80} \\
& MSE$_{\times10^{-3}}$ & \textbf{1.28 $\pm$ 0.02} & 5.09 & 1.62 & 1.64 & 12.36 & 2.46 & 1.37 & 3.33 & 1.45 & \underline{1.36} \\
& RMSE$_{\times10^{-2}}$ & \textbf{3.58 $\pm$ 0.03} & 7.13 & 4.03 & 4.05 & 11.12 & 4.96 & 3.70 & 5.78 & 3.81 & \underline{3.69} \\
& MAPE (\%) & \textbf{4.70 $\pm$ 0.12} & 10.39 & 5.70 & 5.64 & 16.16 & 6.97 & 5.08 & 6.65 & 5.20 & \underline{4.95} \\
\bottomrule
\end{tabular}
\caption{
Performance comparison on four datasets. Results for CompEvo are reported as mean $\pm$ standard deviation over 5 runs. A one-sided Wilcoxon signed-rank test against LTP-LLM and GPT4MTS gives $p<0.05$ on each dataset. Bold and underlined values mark the lowest and second-lowest errors.
}
\label{tab:comparison_results}
\end{table*}

\subsection{Training Objective}

The training objective jointly optimizes forecasting, strategy update, diversity, and pruning, keeping competition differentiable while discouraging similar logic and weak-agent influence.

\vspace{-0.8cm}
\begin{equation}
L = L_{\text{pred}} + \lambda_{\text{PG}} L_{\text{pg}} + \lambda_{\text{diversity}} L_{\text{div}} + \lambda_{\text{pruning}} L_{\text{prune}}.
\label{eq:total_loss}
\end{equation}
where $L_{\text{pred}}$, $L_{\text{pg}}$, $L_{\text{div}}$, and $L_{\text{prune}}$ correspond to forecasting accuracy, strategy update, logic diversity, and weak-agent suppression, respectively. $\lambda_{\text{PG}}$, $\lambda_{\text{diversity}}$, and $\lambda_{\text{pruning}}$ are balancing coefficients. 
The diversity loss is defined below.
\vspace{-0.3cm}
\begin{equation}
L_{\text{div}} = \sum_{i=1}^{I} \sum_{j=i+1}^{I} \frac{h_i^{(e)} \cdot h_j^{(e)}}{\|h_i^{(e)}\|_2 \|h_j^{(e)}\|_2},
\label{eq:diversity_loss}
\end{equation}
where $h_i^{(e)}$ and $h_j^{(e)}$ are the logic representations of agents $i$ and $j$ in round $e$. This term penalizes high cosine similarity to encourage different evidence-seeking patterns. 
The pruning loss is defined below.

\vspace{-0.3cm}
\begin{equation}
L_{\text{prune}} = \sum_{i=1}^{I} |g_i|,
\label{eq:pruning_loss}
\end{equation}
where $g_i$ is the gate in Equation~(\ref{eq:weight_softmax}).

The four terms act at distinct points in the trainable loop. $L_{\text{pred}}$ evaluates the aggregate forecast, $L_{\text{pg}}$ converts relative reward into logic updates, and $L_{\text{div}}$ separates candidate reasoning paths before fusion. Together with gated weighting, $L_{\text{prune}}$ suppresses persistently weak agents without eliminating them. The objective thereby connects forecasting feedback, strategy variation, and differentiable selection.

\subsection{Theoretical Justification}

The theoretical analysis states why the competitive learning process has a stable solution and fitness-based selection is consistent with the intended strategy update, providing guarantees that underpin the stability and reliability of CompEvo. Detailed statements and proofs are provided in Appendix~\ref{sec:theory}.

In our formulation, each agent is treated as a player, its parameterized policy for generating textual logic defines a mixed strategy, and the reward derived from forecasting performance serves as the payoff under partial information. The induced interaction admits a mixed-strategy Bayesian Nash equilibrium (BNE)  (Theorem~\ref{thm:bne_existence}), ensuring that the competitive update process is well-defined even when agents only observe partial information about others.

The gradient dynamics can converge to a local equilibrium under the differentiable objective (Theorem~\ref{thm:convergence}), which supports stable strategy updates through gradient-based optimization. The regret analysis further shows that the competitive update rule is more favorable than debate-based interaction (Theorem~\ref{thm:regret_bound} and Lemma~\ref{lemma:linear_regret}), indicating improved long-term forecasting performance.

The monotonicity result (Proposition~\ref{prop:monotonicity}) confirms that influence weights increase with long-term fitness under the soft selection mechanism, ensuring that strong strategies are retained while weak strategies gradually lose influence. Together, these results justify the main design choices of CompEvo, with fitness-based selection realizing competitive selection, strategy evolution supporting adaptation under incomplete information, and diversity regularization discouraging early textual-logic collapse.


\section{Experiments}
\label{sec:experiments}

\subsection{Experimental Setup}
\textbf{Datasets. } 
We evaluate CompEvo on four NTSF benchmarks: Electricity for energy demand, Exchange for financial markets, Traffic for transportation flows, and Bitcoin for cryptocurrency dynamics. Dataset provenance, scale, and example news inputs are reported in Appendix~\ref{dataset}.

\noindent \textbf{Metrics. } 
We evaluate forecasting quality using mean absolute error (MAE), mean squared error (MSE), root mean squared error (RMSE), and mean absolute percentage error (MAPE) following prior work~\cite{zhou2023ptse,Wang2024}. These metrics measure the average, squared, or percentage deviations between predicted and true values, with lower values indicating better performance. 

In addition, we design LUD to evaluate whether agents maintain diverse strategy updates during training. LUD measures meaningful textual-logic changes rather than superficial rephrasing by combining embedding-level update magnitude with a rubric-based innovation check. For agent $i$ at update step $t$, LUD uses the Rubric-based Innovation Update Score (RIUS) and is computed as follows.
\vspace{-0.2cm}
\begin{equation}
\mathrm{LUD}_i^t = \Delta_i^t \cdot \mathbb{I}(\mathrm{RIUS}_i^t \ge 1),
\end{equation}
where $\Delta_i^t$ measures embedding-level logic change, and RIUS is produced by an LLM judge, with details provided in Appendix~\ref{app:rius_prompt}.

\noindent \textbf{Baselines. } 
The baselines are organized into three categories.
%
(1) \textit{Numerical forecasting baselines}, which rely solely on historical time-series signals, including DLinear, iTransformer, FiLM, Pyraformer, Fedformer, and GPT4TS.
%
(2) \textit{News-aware forecasting baselines}, which explicitly incorporate external news or events for forecasting, including News2Forecast (N2F), GPT4MTS, and LTP-LLM.
%
(3) \textit{Agent-based baselines}, including a single-agent baseline (SA), multi-agent ensemble learning (MAEL), multi-agent evolution with GPT API (MAE-GPT), and multi-agent debate with reinforcement learning (MAD-RL). Detailed descriptions are provided in Appendix~\ref{sec:baselines}. 
%

\noindent \textbf{Implementation. }
We use 10 agents with a shared backbone. Unless otherwise specified, the main tables report this default setting. CompEvo performs direct multi-step forecasting: each call predicts the full 48-step target window in one output, and no predicted value is fed back into a later forecasting call. All compared agent-based methods follow this protocol. Detailed implementation settings are provided in Appendix~\ref{sec:implementation_details}.

\begin{table}[t]
\centering
\small
\setlength{\tabcolsep}{3.5pt}
\renewcommand{\arraystretch}{1.05}
\resizebox{\columnwidth}{!}{%
\begin{tabular}{lclrr}
\toprule
Method & Agents & Model(s) & Trainable params & Train steps \\
\midrule
SA & 1 & Llama-3.1-8B & 41.9M & 2,193 \\
MAEL & 10 & Llama-3.1-8B & 419.4M & 2,193 \\
MAE-GPT & 10 & GPT-4o + Llama-3.1-8B & 419.4M & 2,193 \\
MAD-RL & 10 & Llama-3.1-8B & 1,258.4M & 2,193 \\
CompEvo & 10 & Llama-3.1-8B & 922.8M & 2,193 \\
\bottomrule
\end{tabular}
}
\caption{Agent-based baseline configurations. All methods use a 4,096-token context. Counts exclude frozen backbone parameters and the frozen GPT-4o selector in MAE-GPT.}
\label{tab:baseline_configuration_main}
\end{table}

The comparison is not capacity matched: CompEvo uses 26.7\% fewer trainable parameters than MAD-RL and 2.20$\times$ as many as MAEL or MAE-GPT. The main result therefore compares interaction mechanisms under disclosed, rather than identical, trainable capacity.


\subsection{Main Results}

CompEvo consistently achieves strong performance across all datasets and metrics, demonstrating robust improvements beyond a single evaluation criterion. Table~\ref{tab:comparison_results} reports a full comparison against numerical, news-aware, and LLM-based forecasting baselines. The results show that CompEvo remains competitive on MAE, MSE, RMSE, and MAPE, indicating that the gains are not limited to one metric or one domain. 

Table~\ref{tab:simple_avg_results} isolates the single-agent and multi-agent setting. For each dataset, we compute the relative error reduction over MAD-RL and then report the macro average across the four datasets. Under this protocol, CompEvo lowers RMSE by \textbf{29.8\%} and MAPE by \textbf{30.3\%}. Compared with SA, CompEvo consistently achieves lower errors, highlighting that its gains stem from population-level competition rather than stronger individual agents. It also outperforms MAEL, showing that the improvement arises from competitive evolution rather than simple LoRA-based aggregation. Against MAE-GPT and MAD-RL, CompEvo preserves effective strategies and maintains diverse evidence-seeking patterns more effectively, demonstrating the advantages of end-to-end optimization and competition-driven adaptation.



\begin{table}[htbp]
\centering
\setlength{\tabcolsep}{4pt}
\renewcommand{\arraystretch}{0.5}
\scriptsize
\resizebox{\columnwidth}{!}{%
\begin{tabular}{llccccc}
\toprule
\textbf{Dataset} & \textbf{Metrics} & \textbf{CompEvo} & \textbf{SA} & \textbf{MAEL} & \textbf{MAE-GPT} & \textbf{MAD-RL} \\
\midrule
\multirow{4}{*}{\textbf{Electricity}}
& MAE & \textbf{220.15} & \underline{245.50} & 246.13 & 254.61 & 248.76 \\
& MSE$_{\times10^{3}}$ & \textbf{125.44} & 185.10 & 177.68 & 195.15 & \underline{166.67} \\
& RMSE & \textbf{354.18} & 430.20 & 421.52 & 441.76 & \underline{408.25} \\
& MAPE (\%) & \textbf{6.45} & 7.45 & 7.42 & 7.49 & \underline{6.98} \\

\midrule
\multirow{4}{*}{\textbf{Exchange}}
& MAE$_{\times10^{-3}}$ & \textbf{4.35} & 13.10 & 11.26 & 11.69 & \underline{9.03} \\
& MSE$_{\times10^{-4}}$ & \textbf{0.35} & 17.20 & 12.82 & 18.32 & \underline{8.24} \\
& RMSE$_{\times10^{-2}}$ & \textbf{0.59} & 4.15 & 3.58 & 4.28 & \underline{2.87} \\
& MAPE (\%) & \textbf{0.60} & 5.50 & 4.55 & 5.58 & \underline{3.75} \\

\midrule
\multirow{4}{*}{\textbf{Traffic}}
& MAE$_{\times10^{-2}}$ & \textbf{1.48} & 1.80 & 1.78 & \underline{1.65} & 1.69 \\
& MSE$_{\times10^{-3}}$ & \textbf{0.95} & 1.75 & 1.62 & 1.88 & \underline{1.44} \\
& RMSE$_{\times10^{-2}}$ & \textbf{3.08} & 4.18 & 4.02 & 4.34 & \underline{3.79} \\
& MAPE (\%) & \textbf{3.75} & 4.90 & 4.73 & 5.04 & \underline{4.66} \\

\midrule
\multirow{4}{*}{\textbf{Bitcoin}}
& MAE$_{\times10^{-3}}$ & \textbf{2.65} & 3.14 & 3.09 & \underline{2.82} & 2.94 \\
& MSE$_{\times10^{-3}}$ & \textbf{1.28} & 3.14 & 1.91 & 1.69 & \underline{1.51} \\
& RMSE$_{\times10^{-2}}$ & \textbf{3.58} & 5.60 & 4.37 & 4.11 & \underline{3.89} \\
& MAPE (\%) & \textbf{4.70} & 6.50 & 5.35 & \underline{4.95} & 5.22 \\
\bottomrule
\end{tabular}
}
\caption{Agent baseline comparison.}
\label{tab:simple_avg_results}
\vspace{-0.4cm}
\end{table}

\subsection{Backbone Robustness}
CompEvo remains effective across different backbone models on Electricity, suggesting that its gains are not tied to a single architecture. Table~\ref{tab:model_generalization_summary} reports these results as supplementary robustness evidence.

\begin{table}[htbp]
\centering
\small
\setlength{\tabcolsep}{3pt}
\renewcommand{\arraystretch}{0.9}
\begin{tabular}{lcccc}
\toprule
\textbf{Model} & \textbf{RMSE} & \textbf{MSE} & \textbf{MAE} & \textbf{MAPE (\%)} \\
\midrule
Mistral-v0.1-7B & 366.84 & 134.57 & 245.13 & 7.18 \\
Gemma-2-9B & 364.12 & 132.58 & 228.14 & 6.69 \\
Llama-2-7B & 368.53 & 135.81 & 248.79 & 7.25 \\
Llama-3.1-8B & 354.18 & 125.44 & 220.15 & 6.45 \\
Llama-3.1-70B & 342.57 & 117.35 & 212.08 & 6.22 \\
Qwen-2.5-7B & 363.27 & 131.97 & 227.43 & 6.68 \\
Qwen-2.5-32B & 349.83 & 122.38 & 218.09 & 6.38 \\
\bottomrule
\end{tabular}
\caption{Backbone robustness on Electricity. MSE is reported in $10^3$, and MAPE is reported in percentage.}
\vspace{-0.4cm}
\label{tab:model_generalization_summary}
\end{table}

\begin{table*}[htbp]
\centering
\scriptsize
\renewcommand{\arraystretch}{0.5}
\resizebox{\textwidth}{!}{%
\begin{tabular}{@{}l|cccc|cccc@{}}
\cmidrule(r){1-9}
\multicolumn{1}{l|}{\multirow{2}{*}{\textbf{Model Variant}}} & \multicolumn{4}{c|}{\textbf{Electricity}} & \multicolumn{4}{c}{\textbf{Exchange}} \\ 
\cmidrule(lr){2-9}
\multicolumn{1}{l|}{} & RMSE & MSE\( \times 10^{3} \) & MAE & MAPE\( (\%) \) & RMSE\( \times 10^{-2} \) & MSE\( \times 10^{-4} \) & MAE\( \times 10^{-3} \) & MAPE\( (\%) \) \\ 
\cmidrule(r){1-9}
Simplified Logic Gen & 445.53 & 198.50 & 255.24 & 7.65 & 4.10 & 16.81 & 14.84 & 4.85 \\
Uniform Weights       & 452.12 & 204.41 & 265.53 & 7.93 & 4.82 & 23.23 & 20.53 & 5.65 \\
w/o $L_{\text{diversity}}$ & 410.56 & 168.56 & 240.20 & 7.05 & 3.88 & 15.05 & 9.50  & 3.84 \\
w/o $L_{\text{pruning}}$   & 405.22 & 164.20 & \underline{238.51} & \underline{6.95} & 3.76 & 14.13 & 9.81  & 3.92 \\
w/o GRPO        & \underline{388.51} & \underline{150.94} & 242.08 & 7.12 & \underline{1.82} & \underline{3.31} & \underline{7.97} & \underline{2.82} \\
CompEvo & \textbf{354.18} & \textbf{125.44} & \textbf{220.15} & \textbf{6.45} & \textbf{0.59} & \textbf{0.35} & \textbf{4.35} & \textbf{0.60} \\
\cmidrule(r){1-9}

\multirow{2}{*}{} & \multicolumn{4}{c|}{\textbf{Traffic}} & \multicolumn{4}{c}{\textbf{Bitcoin}} \\ 
\cmidrule(lr){2-9}
\multicolumn{1}{l|}{} & RMSE\(\times 10^{-2} \) & MSE\( \times 10^{-3} \) & MAE\( \times 10^{-2} \) & MAPE\( (\%) \) 
& RMSE\( \times 10^{-2} \) & MSE\( \times 10^{-3} \) & MAE\( \times 10^{-3} \) & MAPE\( (\%) \) \\ 
\cmidrule(r){1-9}
Simplified Logic Gen & 4.29 & 1.84 & 1.92 & 5.46 & 4.13 & 1.71 & 3.11 & 6.84 \\
Uniform Weights      & 4.75 & 2.26 & 2.25 & 6.07 & 4.45 & 1.98 & 3.37 & 7.81 \\
w/o GRPO       & \underline{3.54} & \underline{1.25} & 1.70 & 4.36 & \underline{3.68} & \underline{1.35} & \underline{2.79} & 5.18 \\
w/o $L_{\text{diversity}}$ & 3.95 & 1.56 & \underline{1.68} & \underline{4.30} & 3.85 & 1.48 & 2.86 & \underline{5.11} \\
w/o $L_{\text{pruning}}$   & 4.02 & 1.62 & 1.75 & 4.52 & 3.95 & 1.56 & 2.97 & 5.36 \\
CompEvo & \textbf{3.08} & \textbf{0.95} & \textbf{1.48} & \textbf{3.75} & \textbf{3.58} & \textbf{1.28} & \textbf{2.65} & \textbf{4.70} \\
\cmidrule(r){1-9}
\end{tabular}
}
\caption{Architectural ablations across four datasets. Each variant changes one component while retaining the remaining model and training setup.}
\label{tab:ablation}
\end{table*}

\subsection{Ablation Study}

Table~\ref{tab:ablation} evaluates five variants under the same backbone, data splits, and training setup. Each variant changes one component; coefficient sensitivity is analyzed separately in Section~\ref{sec:hyperparameter_sensitivity}.

\begin{itemize}
    \setlength{\itemsep}{0pt}
    \setlength{\parskip}{0pt}
    \setlength{\parsep}{0pt}
    \item \textbf{Simplified Logic Gen:} replaces the three-stage opponent-aware update with a simplified logic generator.
    \item \textbf{Uniform Weights:} replaces fitness-based selection with equal aggregation ($w_i=1/I$).
    \item \textbf{w/o $L_{\text{diversity}}$:} removes the penalty that prevents agents from learning similar logic.
    \item \textbf{w/o $L_{\text{pruning}}$:} removes the suppression of persistently weak agents.
    \item \textbf{w/o GRPO:} removes the group-relative policy update for logic generation.
\end{itemize}

Every variant performs worse than CompEvo across all datasets and metrics. Simplified Logic Gen and Uniform Weights cause the largest overall drops, highlighting the importance of opponent-aware logic evolution and fitness-based selection. The other three variants show that diversity control, pruning, and GRPO provide additional, consistent gains.
Because the backbone and temporal split remain unchanged, these results also suggest that the gains come from competitive adaptation rather than static memorization.

\noindent \textbf{Leakage-controlled inference check.}
We additionally evaluate the trained model without further updates on 100 Electricity samples from February 19 to March 30, 2026, using AEMO demand data aligned with GDELT news. This temporally later set follows the original 48-step input/output construction and is excluded from training, validation, model selection, and hyperparameter tuning.

\begin{table}[t]
\centering
\small
\setlength{\tabcolsep}{5pt}
\begin{tabular}{lcc}
\toprule
Metric & Original test & 2026 inference \\
\midrule
MSE & 125440.00 & 111141.83 \\
RMSE & 354.18 & 333.38 \\
MAE & 220.15 & 225.02 \\
MAPE (\%) & 6.45 & 6.99 \\
\bottomrule
\end{tabular}
\caption{Post-training inference on temporally later Electricity samples.}
\label{tab:memorization_risk_2026_main}
\end{table}

Table~\ref{tab:memorization_risk_2026_main} shows comparable performance, with RMSE 333.38 and MAPE 6.99\%. This sanity check argues against a memorization-only explanation, while it does not replace evaluation on continuously updated news streams. Additional protocol details are provided in Appendix~\ref{subsec:performance_source_discussion}.

\subsection{Optimization Stability and Efficiency}

\begin{figure}[htbp]
    \centering
    \includegraphics[width=0.48\textwidth]{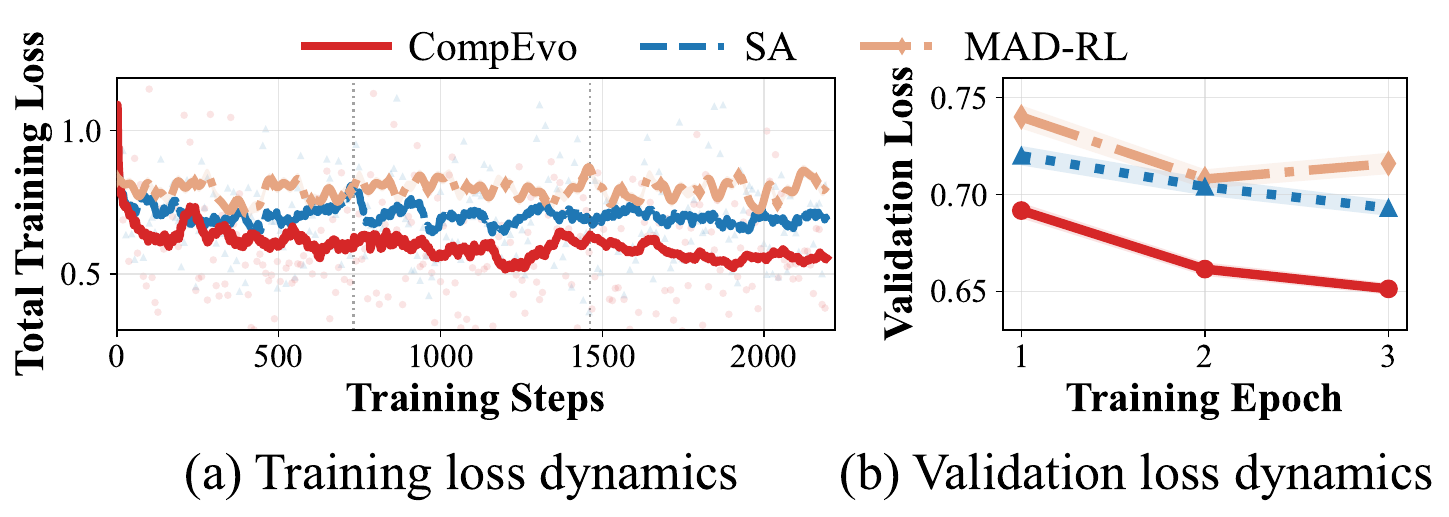}
    \caption{Training dynamics on the Electricity dataset.}
    \label{fig:convergence_loss}
\end{figure}

CompEvo follows a faster and more stable convergence trajectory than SA and MAD-RL while keeping multi-agent computational overhead moderate. As shown in Figure~\ref{fig:convergence_loss}(a), CompEvo's training loss drops more rapidly and remains consistently lower, while SA converges slowly and MAD-RL exhibits pronounced oscillations. Figure~\ref{fig:convergence_loss}(b) shows a similar trend on validation loss, indicating consistent generalization without overfitting. This pattern supports the optimization claim and aligns with Theorem~\ref{thm:convergence}, as updates are guided by a differentiable objective rather than heuristic or prompt-based rules.

\begin{figure}[htbp]
    \centering
    \includegraphics[width=0.38\textwidth]{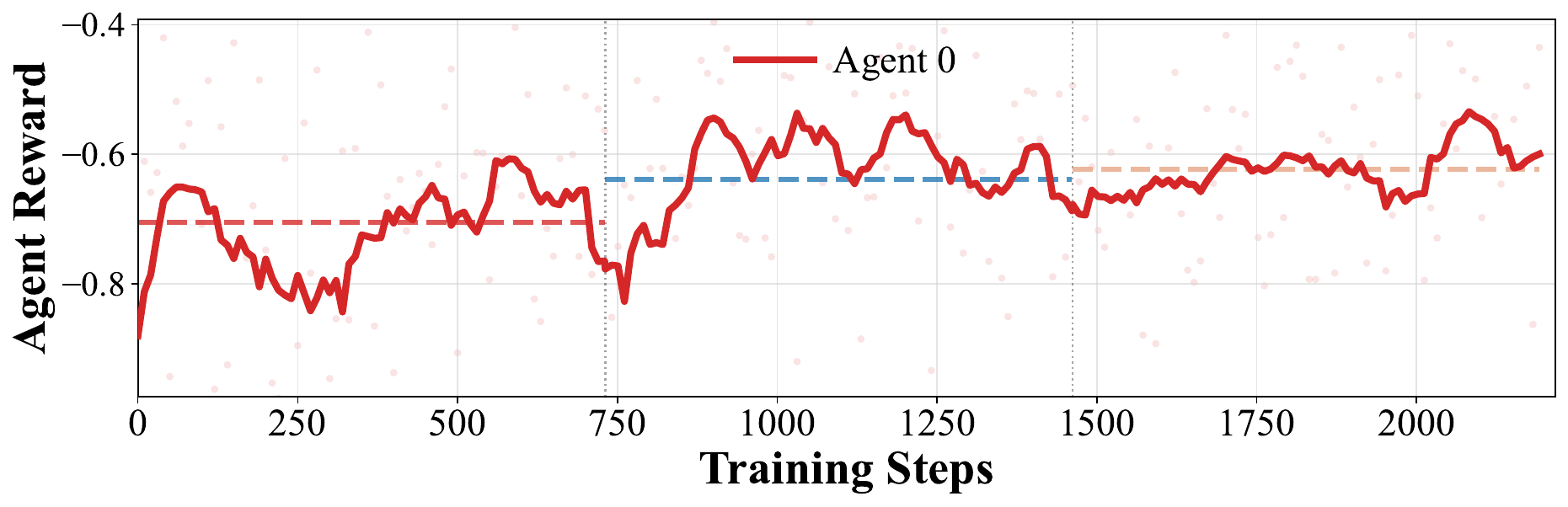}
    \caption{Agent reward during training, where the plateau aligns with the Nash-equilibrium interpretation.}
    \label{fig:agent_reward}
\end{figure}

Figure~\ref{fig:agent_reward} shows a representative agent's reward trajectory. After the initial improvement, the reward plateaus, and further updates produce minimal changes once other agents' strategies and the ensemble are fixed. This indicates that the agent has reached a stable strategy, consistent with a Nash equilibrium as stated in Theorem~\ref{thm:bne_existence}. All agents follow similar reward trends but maintain diverse evidence-seeking strategies, allowing the ensemble to leverage complementary information.

Table~\ref{tab:efficiency_comparison_main} reports a measured efficiency profile on Electricity using three agents to isolate interaction overhead. CompEvo requires 14.5 hours per epoch, 8.57 seconds per inference step, and 16.4 GB of peak memory. Its inference cost lies between MAEL, which averages independent forecasts, and MAD-RL, which performs multi-turn discussion. These measurements characterize the controlled three-agent setting and do not estimate the cost of the default 10-agent configuration.

\begin{table}[htbp]
\centering
\small
\setlength{\tabcolsep}{6pt}
\renewcommand{\arraystretch}{1.05}
\resizebox{\columnwidth}{!}{%
\begin{tabular}{lccc}
\toprule
\textbf{Metric} & \textbf{MAEL} & \textbf{MAD-RL} & \textbf{CompEvo} \\
\midrule
Train time (h/ep) & 11.4 & 40.0 & 14.5 \\
Latency (s/step) & 5.61 & 19.70 & 8.57 \\
Peak memory (GB) & 14.8 & 24.6 & 16.4 \\
\bottomrule
\end{tabular}
}
\caption{Measured efficiency on Electricity under the three-agent setting.}

\label{tab:efficiency_comparison_main}
\end{table}

\subsection{Strategy-Update Diversity and Specialization}
CompEvo maintains stronger strategy-update diversity and clearer agent specialization than collaboration-based baselines. This analysis directly supports our claim that CompEvo maintains diverse and specialized agent behaviors. We measure these properties using LUD and the number of unique news items selected by each agent. As shown in Figure~\ref{fig:logic_evolution_specialization}(a), CompEvo keeps a higher LUD across epochs, while removing diversity regularization or using MAD leads to faster LUD decay. This suggests that diversity-controlled competition prevents premature convergence and preserves heterogeneous reasoning updates.

Figure~\ref{fig:logic_evolution_specialization}(b) further supports agent specialization. CompEvo agents select more unique news items and develop distinct topical focuses, such as regulation, climate, infrastructure, and supply chain. Fitness-based competition preserves some differentiation even without diversity regularization, whereas MAD's consensus updates quickly flatten diversity.

\begin{figure}[htbp]
  \centering
  \includegraphics[width=1\columnwidth]{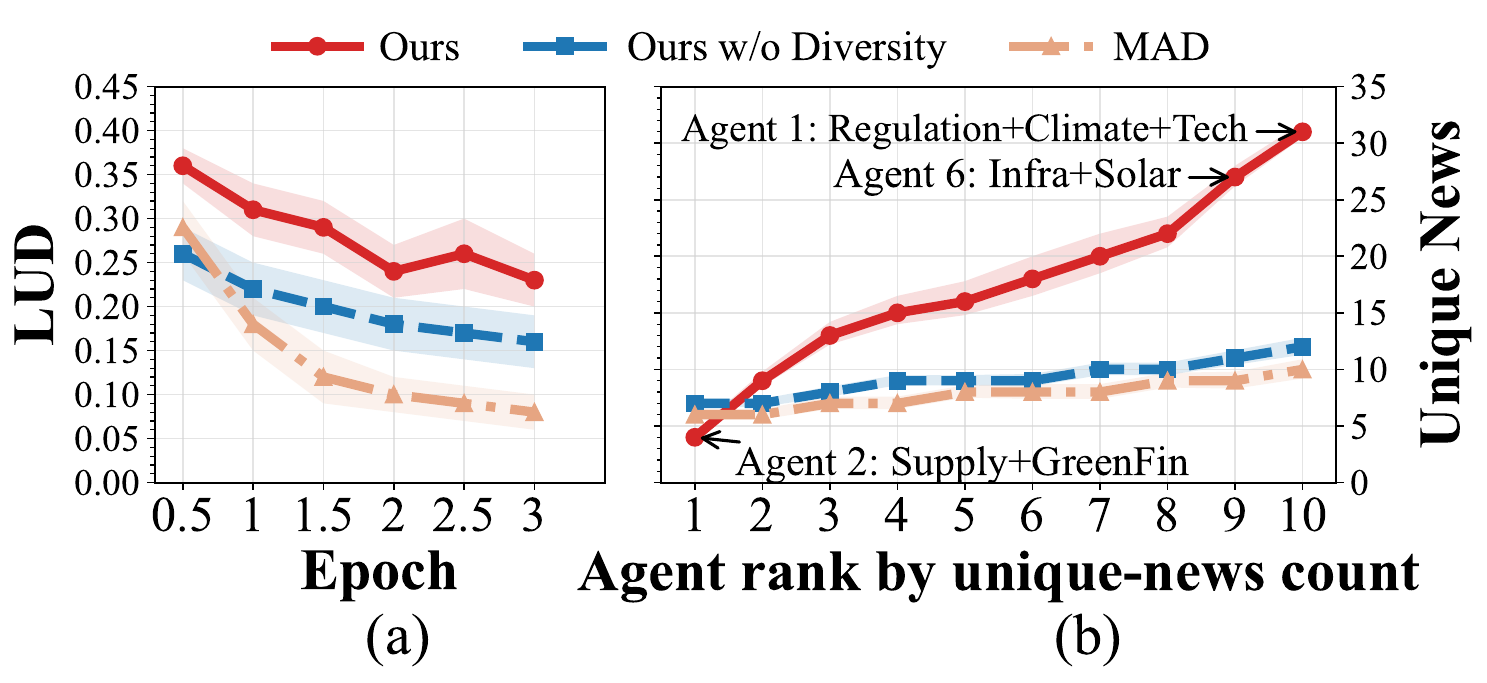}
  \caption{Strategy-update diversity and agent specialization. 
(a) Average LUD per agent across epochs. 
(b) Rank plot of unique news counts over 100 steps.}
\label{fig:logic_evolution_specialization}
\end{figure}

Three appendix checks support this mechanism-level reading. The embedding visualization in Section~\ref{sec:mechanism_visualization} shows clearer logic separation under diversity regularization, higher LUD is associated with larger error reduction (Table~\ref{tab:lud_error_correlation}), and the LUD ranking remains stable across models (Table~\ref{tab:judge_robustness}).

\subsection{Hyperparameter Sensitivity}
\label{sec:hyperparameter_sensitivity}

Unlike the architectural removals in Table~\ref{tab:ablation}, this analysis retains both regularizers and varies only their coefficients. Both diversity and pruning regularization work best at moderate strength.
Figure~\ref{fig:sensitivity_analysis_main} shows representative sweeps on Electricity and Exchange. 
Diversity regularization reaches the best balance around $\lambda_{\text{diversity}}=0.1$, where RMSE is low while logic similarity remains controlled. 
Pruning regularization performs best around $\lambda_{\text{pruning}}=0.01$, suppressing weak agents without overly shrinking the population. 
Similar trends on the remaining datasets and additional hyperparameter results are provided in Appendix~\ref{sec:LLMs_tests}.



\vspace{-10pt}
\begin{figure}[htbp]
  \centering
  \includegraphics[width=\columnwidth]{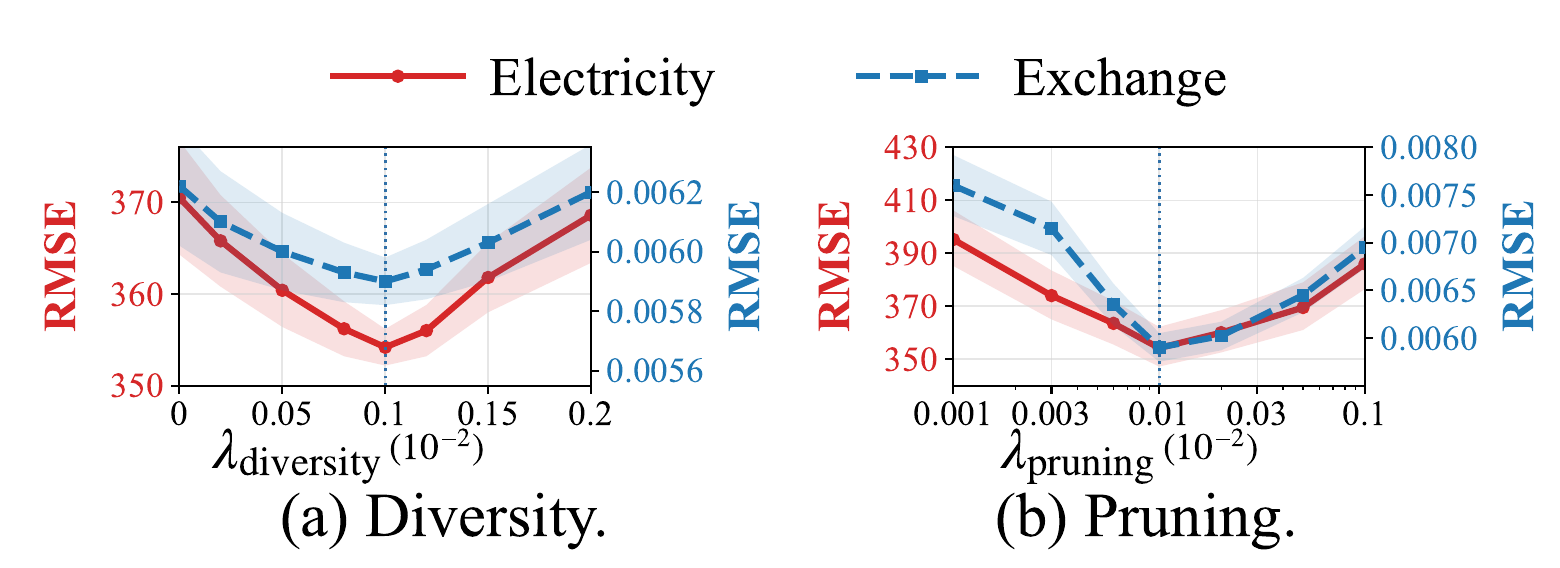}
\caption{Diversity and pruning regularization.}
  \label{fig:sensitivity_analysis_main}
\end{figure}
\vspace{-5pt}

\subsection{Robustness  under News Perturbations}
\label{sec:noisy_news_main}

CompEvo shows robust performance under news perturbations. Figure~\ref{fig:noisy_news_robustness_main} presents RMSE on the Electricity and Exchange datasets under different perturbations: shuffled news order (Shuffle), irrelevant news injection (Irrel.), and noisy news injection (Noisy), compared to clean input (Orig.). CompEvo consistently degrades more slowly than MAEL, MAE-GPT, and MAD-RL, indicating that fitness-based selection effectively down-weights strategies that overreact to misleading news. Similar patterns are observed on the other datasets (see Appendix~\ref{sec:noisy_news}).

\begin{figure}[htbp]
  \centering
  \includegraphics[width=\columnwidth]{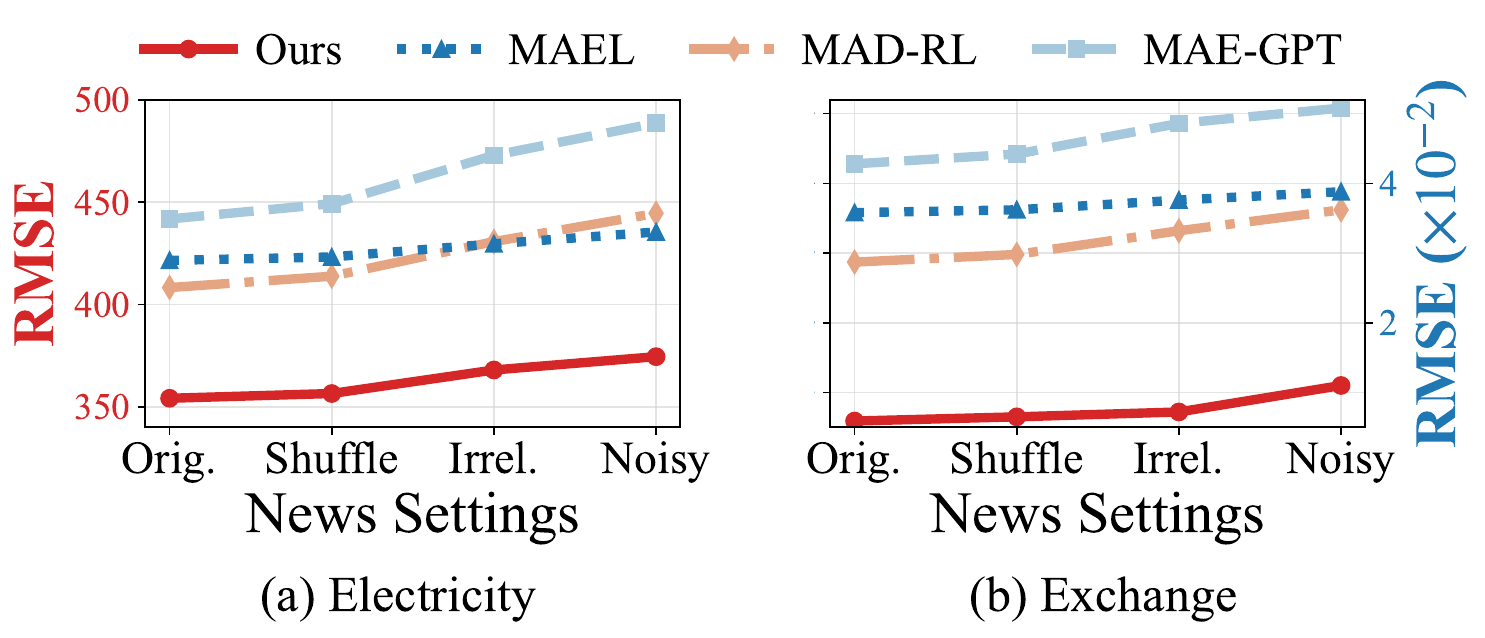}
\caption{RMSE under news perturbations.}

  \label{fig:noisy_news_robustness_main}
\end{figure}
\vspace{-10pt}



\subsection{Strategy Diversity Visualization}
\label{sec:mechanism_visualization}

We visualize the final-epoch hidden logic embeddings of 10 agents in 3D (Figure~\ref{fig:logic_embedding_appendix}). CompEvo agents spread broadly, specializing in distinct semantic areas such as Regulation+Climate+Tech, Supply+GreenFin, and Infra+Solar, while MAD-RL agents cluster tightly, indicating collapsed strategies. This shows that diversity regularization and competitive evolution enable heterogeneous reasoning and specialization, whereas debate-based coordination leads to homogenized embeddings.

\begin{figure}[htbp]
  \centering
  \includegraphics[width=0.7\columnwidth]{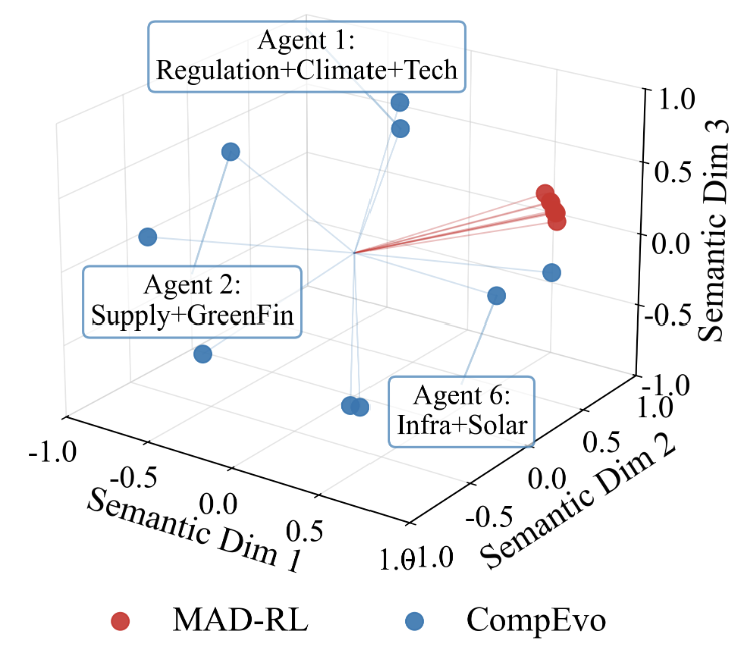}
    \caption{
    Agent strategy embeddings, red (MAD-RL) and blue (CompEvo).
    }
    \vspace{-0.3cm}
  \label{fig:logic_embedding_appendix}
\end{figure}



\section{Conclusions}
In this paper, we propose CompEvo, a competition-induced evolution framework for multi-agent news-driven time series forecasting. 
Inspired by evolutionary game theory, CompEvo formulates agent competition as a unified differentiable objective, enabling agents to adapt through end-to-end training. 
Experimental results show that CompEvo improves forecasting accuracy while maintaining diverse and specialized agent strategies.
Ablation, convergence, and perturbation analyses link these gains to fitness-guided specialization that remains robust across backbones and evidence quality.

\clearpage 

\section*{Limitations}
CompEvo is validated on NTSF benchmarks, where external textual events are directly relevant to future numerical values. Extending the framework to more general forecasting tasks or broader decision-making scenarios is a promising direction for future work.

This work focuses on textual news evidence. Real-world event signals may also include images, videos, audio, and structured multimodal records. Incorporating such multimodal evidence could further broaden the applicability of competition-induced multi-agent forecasting.


\bibliography{custom}


\ifwithappendix
\clearpage
\appendix

\phantomsection
\begin{center}
{\Large Contents of Appendix}
\end{center}
\vspace{0.2cm}

\startcontents[appendix]
\printcontents[appendix]{l}{1}{\setcounter{tocdepth}{2}}

\vspace{1 cm}

\section{Notation}
\label{sec:notation}

\noindent\textbf{Task and agent notation.}

\begin{center}
\scriptsize
\setlength{\tabcolsep}{3pt}
\renewcommand{\arraystretch}{1.08}
\begin{tabularx}{\linewidth}{p{0.25\linewidth}X}
\toprule
\textbf{Symbol} & \textbf{Description} \\
\midrule
$\mathbf{x}_{t-L:t-1}$ & Historical numerical observations from time $t-L$ to $t-1$. \\
$L$ & Length of the historical observation window. \\
$\mathcal{C}_t$ & Time-aligned public news pool at forecasting step $t$. \\
$x_t$ & Ground-truth numerical value at forecasting step $t$. \\
$\hat{x}_t$ & Predicted numerical value at forecasting step $t$. \\
$\mathcal{F}$ & Forecasting mapping from historical observations and news to prediction. \\
$I$ & Number of agents in the population. \\
$i,j$ & Agent indices. \\
$\mathcal{C}_{t,i}$ & News subset selected by agent $i$. \\
$\hat{x}_{t,i}$ & Individual prediction produced by agent $i$. \\
$e$ & Strategy-update round. \\
$E$ & Total number of strategy-update rounds. \\
$o_i^{(e)}$ & Observable state of agent $i$ in round $e$. \\
$X_s$ & Historical time-series context for sample $s$. \\
$\mathcal{C}_s$ & Shared candidate news pool for sample $s$. \\
$l_i^{(e)}$ & Textual logic generated or used by agent $i$ in round $e$. \\
$l_{-i}^{(e)}$ & Exposed textual logic of agents except agent $i$. \\
\bottomrule
\end{tabularx}
\end{center}

\vspace{0.2cm}
\noindent\textbf{Strategy evolution and selection notation.}

\begin{center}
\scriptsize
\setlength{\tabcolsep}{3pt}
\renewcommand{\arraystretch}{1.08}
\begin{tabularx}{\linewidth}{p{0.25\linewidth}X}
\toprule
\textbf{Symbol} & \textbf{Description} \\
\midrule
$\pi_i(a \mid o_i,\theta_i)$ & Parameterized strategy policy of agent $i$. \\
$a$ & Strategy action sampled from the policy and realized as textual logic. \\
$\theta_i$ & Trainable parameters of agent $i$. \\
$h_i^{(e)}$ & Hidden representation of agent $i$'s textual logic. \\
$\bar{h}_i^{(e)}$ & Candidate logic representation generated by agent $i$. \\
$c_i^{(e)}$ & Opponent context aggregated from peer logic representations. \\
$\alpha_{ij}$ & Interaction weight from agent $j$ to agent $i$. \\
$g_{2,i}, g_{3,i}$ & Gating variables in the logic-evolution process. \\
$W_{g2}, W_{g3}$ & Trainable projection matrices for gated logic fusion. \\
$b_{g2}, b_{g3}$ & Trainable bias terms for gated logic fusion. \\
$\odot$ & Element-wise multiplication. \\
$\sigma(\cdot)$ & Sigmoid activation function. \\
$R_i^{(e)}$ & Immediate reward of agent $i$ in round $e$. \\
$M_i^{(e)}$ & Long-term fitness of agent $i$ in round $e$. \\
$\beta$ & EMA coefficient for reward smoothing. \\
$g_i$ & Learnable gate in fitness-based selection. \\
$\tau$ & Temperature parameter in the soft selection rule. \\
$w_i$ & Aggregation weight assigned to agent $i$. \\
$\hat{\mathbf{Y}}_{i,s}^{(e)}$ & Individual 48-step prediction vector of agent $i$. \\
$\hat{\mathbf{Y}}_{\mathrm{agg}}$ & Fitness-weighted aggregated prediction vector. \\
$\mathbf{Y}_s$ & Ground-truth 48-step target vector for sample $s$. \\
\bottomrule
\end{tabularx}
\end{center}

\vspace{0.2cm}
\noindent\textbf{Optimization and theoretical notation.}

\begin{center}
\scriptsize
\setlength{\tabcolsep}{3pt}
\renewcommand{\arraystretch}{1.08}
\begin{tabularx}{\linewidth}{p{0.25\linewidth}X}
\toprule
\textbf{Symbol} & \textbf{Description} \\
\midrule
$\mathcal{L}_{\mathrm{metric}}$ & Differentiable forecasting metric used for reward and prediction loss. \\
$L_{\mathrm{pred}}$ & Forecasting loss of the aggregated prediction. \\
$L_{\mathrm{pg}}$ & Policy-gradient loss for strategy evolution. \\
$L_{\mathrm{total}}$ & Overall training objective. \\
$\lambda_{\mathrm{PG}}$ & Weight of the policy-update loss. \\
$\lambda_{\mathrm{diversity}}$ & Weight of the diversity regularization term. \\
$\lambda_{\mathrm{pruning}}$ & Weight of the pruning regularization term. \\
$\Pi_i$ & Strategy space of agent $i$. \\
$\pi_i^*$ & Equilibrium strategy of agent $i$. \\
$\pi_{-i}$ & Strategies of all agents except agent $i$. \\
$BR_i$ & Best-response correspondence of agent $i$. \\
$V_i^\pi$ & Value function of agent $i$ under policy $\pi$. \\
$Q_i^\pi$ & Q-function of agent $i$ under policy $\pi$. \\
$\gamma$ & Discount factor for expected future rewards. \\
$RE_i(T)$ & Bayesian regret of agent $i$ over $T$ steps. \\
$\epsilon_t$ & Estimation error term in the regret analysis. \\
$\Delta_t$ & Policy suboptimality term in the regret analysis. \\
\bottomrule
\end{tabularx}
\end{center}

\section{Theoretical Justification}
\label{sec:theory}
This appendix provides theoretical support for our EGT-inspired framework, covering the existence of a stable equilibrium, local convergence of the differentiable dynamics, and sublinear regret under standard regularity assumptions. We also compare the framework against adversarial debate dynamics.

\subsection{Existence of a BNE}
\label{sec:bne_proof}

We theoretically support the stability of our multi-agent framework by establishing the existence of a BNE, where no agent can improve its expected payoff by unilaterally changing its strategy, given the strategies of others. The existence of such an equilibrium suggests that the competition-driven evolutionary process can admit a stable strategy profile, as supported by the conditions of Glicksberg's fixed point theorem~\cite{glicksberg1952further}.

\begin{define}[Best Response Correspondence]
For each agent $i$, its best response correspondence $BR_i$ is defined as
\begin{equation}
\begin{aligned}
    BR_i(\pi_{-i}) = & \{ \pi_i \in \Pi_i \mid \\
    & \pi_i \text{ maximizes } R_i(\theta_i, \pi_i, \pi_{-i}) \}
\end{aligned}
\end{equation}
where $\Pi_i$ is the strategy space for agent $i$, and $\pi_{-i}$ denotes the strategies of all other agents.
\end{define}

\begin{theorem}[Existence of BNE]
\label{thm:bne_existence}
In the multi-agent competition framework, a BNE strategy profile $\pi^* = (\pi_1^*, \dots, \pi_N^*)$ exists.
\end{theorem}

\begin{proof}
We examine the three conditions of Glicksberg's theorem and their connections to each agent in the framework:

\noindent Condition 1: The Strategy Space $\Pi_i$ is non-empty, convex, and compact in the topology of convergence.
    \begin{itemize}
        \item Non-empty: the strategy space contains at least one feasible strategy. 
        \item Framework connection: the strategy space $\Pi_i$ consists of all possible strategies $\pi_{i}$ mapping parameters $\theta_i$ of agent $i$ to its action $a_i$. 
        \item Compact: the strategy space is bounded and closed.
        \item Framework connection: since the action $\mathbb{A}_i$ and the parameter space $\Theta_i$ are closed and bounded subsets of a finite-dimensional vector space, the two spaces are compact metric spaces, strategy $\pi_{i}(a|o,\theta)$ is considered as a mapping function from two compact spaces. Thus, the strategy space $\Pi_i$ is also compact based on Tychonoff Theorem.  
        \item Convex: the linear combination of any two strategies still belongs to the strategy space.
        \item Framework connection: in the model, an agent $i$ essentially selects an action $a_i$ according to the probability distribution defined by its policy (mixed strategy). Thus, the strategy space is convex, and any convex combination in the space is measurable.
    \end{itemize}

\noindent Condition 2: The payoff function $R_i$ is continuous for fixed parameter $\theta_i$.
    \begin{itemize}
        \item Framework connection: the payoff function depends continuously on strategies $\pi_{i}$ and $\pi_{-i}$ due to the continuity of the function in actions when $\theta_i$ is fixed. 
    \end{itemize}

\noindent Condition 3: The payoff function $R_i$ is quasi-concave in $\pi_i$ for fixed $\theta_i$ and $\pi_{-i}$.
    In our framework, for $\theta_i$ and $\pi_{-i}$ fixed within a single optimization iteration, we operate under a standard local convex relaxation around the observed action subspace. Under this localized approximation, the payoff function $R_i$ behaves locally as quasi-concave with respect to action $a_i$, ensuring the local properties required for the strategy profiles.

\noindent Conclusion: the framework design satisfies the three conditions of Glicksberg's fixed-point theorem and supports the existence of BNE strategy profiles. This provides a theoretical explanation for the stability of our evolutionary system.
\end{proof}

\subsection{Local convergence analysis}
\label{sec:converge_proof}
\begin{theorem}[Local convergence under regularity assumptions]
\label{thm:convergence}
Under the stated compactness, smoothness, and bounded-variance assumptions, minimizing $L_{\text{total}}$ converges to a stationary point that corresponds to a local equilibrium of the induced multi-agent optimization game.
\end{theorem}

\begin{proof}
Theorem~\ref{thm:bne_existence} establishes the existence of a stable equilibrium profile, which indicates that the solution space admits stable points and possibly multiple local optima \cite{Ratliff2013CharacterizationAC}. Because the parameter update process of LLMs involves non-convex value approximation, finding a global optimum in finite time is intractable. We therefore focus on local convergence under regularity assumptions \cite{Borkar2009, yi2025debateequilibriumbeliefdrivenmultiagent}. 

Assume all the parameters of agents' strategies are $\theta = \{\theta_1, \theta_2, ..., \theta_N\}$, where $N$ is the number of agents. The proof of Theorem~\ref{thm:convergence} can be stated as follows. The parameter $\theta$ can be optimized via the gradient of $L_{\text{total}}$ to approach a stationary point $\theta^{*}$, and this point supports a local equilibrium interpretation.

\noindent Step 1: Formalization of the problem and objectives.
    \begin{itemize}
    \item Strategy space: $\Pi = \Pi_{1} \times \Pi_{2} \times \dots \times \Pi_{N}$, where $\Pi_{i}$ is the strategy space of agent $i$ parameterized by $\theta_i$.
    \item Payoff function: $R_i(\theta_i,\theta_{-i})$ represents the reward of agent $i$ based on its parameter $\theta_i$ when the parameters $\theta_{-i}$ of other agents are fixed.
    \item Local equilibrium condition: for agent $i$ at an equilibrium point $\theta_{i}^{*}$, no unilateral local parameter update improves the induced payoff given $\theta_{-i}^{*}$. In first-order form, this corresponds to $\nabla_{\theta_i} R_i(\theta_{i}^{*}, \theta_{-i}^{*}) = 0$.
    \end{itemize}
    
\noindent Step 2: Find the local stationary point.
The strategy parameter update formula based on gradient optimization at the $t$th iteration is

\begin{equation}
\theta_{t+1} = \theta_t - \eta \times \nabla_{\theta} L_{total}(\theta_t)
\end{equation}

\noindent where $\eta$ is the learning rate, $\nabla$ is the gradient of $L_{\text{total}}$. To analyze convergence, we first show the existence of a local stationary point and then relate this point to the induced local equilibrium.

Let the parameter $\theta^{*}$ satisfy $\nabla_{\theta}L_{total}(\theta^{*})=0$, where $\nabla$ indicates the gradient. According to the definition of $L_{total}$ in Equation~(\ref{eq:total_loss}), for each agent $i$, the formula can be expressed as

\begin{equation}
\begin{aligned}
\nabla_{\theta_i}L_{total(\theta^{*})}= & \nabla_{\theta_i}L_{predict}(\theta^{*}) + \\
& \lambda_{PG} \times \nabla_{\theta_i}L_{PG}(\theta^{*}) + \\ 
& \lambda_{strategy} \times \nabla_{\theta_i}L_{strategy}(\theta^{*})=0
\end{aligned}
\end{equation}

\noindent where $L_{predict}$ is directly used to optimize the predicted value, making it as close as possible to the true value. $L_{PG}$ is used to update agent $i$'s strategy based on its parameter $\theta_i$. As summarized in Equation~(\ref{eq:total_loss}), $L_{strategy}$ consists of the diversity and pruning terms. These two regularizers encourage each agent to generate more diverse decision logic while suppressing the interference of incorrect logic on the prediction.

Under the bounded adapter parameterization and standard weight regularization, the local optimization region can be treated as compact and lower bounded. By the Extreme Value Theorem, a local minimum exists in this region. According to Fermat's theorem, the gradient is zero at an interior local minimum. Thus, when $\nabla L_{\text{total}}(\theta^{*})=0$ and $L_{\text{total}}(\theta^{*})$ cannot be locally decreased, $\theta^{*}$ is a local stationary solution.

Mechanism Analysis of the Role of $L_{\text{strategy}}$. We expand the magnitude of the gradient term with respect to $L_{\text{total}}$

\begin{equation}
\begin{aligned}
||\nabla_{\text{total}}||^2  = & ||\nabla_{\text{main}} + \nabla_{\text{strategy}}||^2 \\
 = & ||\nabla_{\text{main}}||^2 + ||\nabla_{\text{strategy}}||^2 + \\
& 2\langle\nabla_{\text{main}}, \nabla_{\text{strategy}}\rangle
\end{aligned}
\end{equation}

\noindent where $\nabla_{\text{main}}$ includes the gradients of both $L_{\text{predict}}$ and $L_{\text{PG}}$. $\langle..\rangle$ means inner product operation (cos similarity). $||\nabla_{\text{main}}||^2$ and $||\nabla_{\text{strategy}}||^2$ are greater than 0. For the inner product term, we consider two cases.

Case 1: $\langle\nabla_{\text{main}}, \nabla_{\text{strategy}}\rangle > 0$. In this case, the inner product term positively contributes to the convergence of the loss function, enabling $L_{\text{diversity}}$ to help the model escape local minima or saddle points and $L_{\text{pruning}}$ to induce sparsity.

Case 2: $\langle\nabla_{\text{main}}, \nabla_{\text{strategy}}\rangle < 0$. In this case, there exists a gradient conflict because $L_{\text{strategy}}$ pushes parameters toward high-entropy regions, while $L_{\text{predict}}$ and $L_{\text{PG}}$ favor lower-entropy predictive updates. Following prior analysis of gradient conflict \cite{NEURIPS2020_3fe78a8a}, the primary task remains stable when the following condition is satisfied:

\begin{equation}
\label{eq23g}
||\nabla_{\text{main}}|| > -||\nabla_{\text{strategy}}|| \cdot \frac{\langle\nabla_{\text{main}}, \nabla_{\text{strategy}}\rangle}{||\nabla_{\text{main}}||\nabla_{\text{strategy}}||}
\end{equation}

The hyperparameters $\lambda_{\text{diversity}}$ and $\lambda_{\text{pruning}}$ regulate the magnitude of $\nabla_{\text{strategy}}$, thereby mitigating gradient conflicts defined in Eq.(\ref{eq23g}), and dampen optimization oscillations. Their tuning, combined with a gradient decay schedule $\eta_{t} = \eta_{0}/(\gamma\sqrt{t})$, lets diversity influence early exploration while later stages prioritize stable convergence.

\noindent Step 3: Construct a Lyapunov function. The expectation of $L_{\text{total}}(\theta_{t})$ can be expressed as

\begin{equation}
\begin{aligned}
E_{\theta}(\Delta L) &= E_{\theta}(L_{\text{total}}(\theta_{t+1}) - L_{\text{total}}(\theta_{t})) \\
& \approx - \eta_{t} \times E_{\theta}(||\nabla_{\theta} L_{\text{total}}(\theta_{t})||^2)
\end{aligned}
\end{equation}

\noindent where $||\nabla_{\theta} L_{\text{total}}(\theta_{t})||^2$ is derived from the first-order Taylor expansion of $L_{\text{total}}(\theta_{t+1}) - L_{\text{total}}(\theta_{t})$ at the $t$th iteration. $E_{\theta}(\Delta L)$ is smaller than 0, indicating that the loss $L_{\text{total}}$ is monotonically decreasing in expectation. So we can directly use the total loss function as Lyapunov function $L_{\text{total}}(\theta)$.

\noindent Step 4: Apply LaSalle's invariance principle.
For the Lyapunov function, we define the set of minimizer points where all gradients are zero as $\Omega=\{\theta \in CD | \nabla_{\theta}L_{\text{total}}(\theta)=0\}$, where $CD$ is the compact set of solution space for $\theta$. Thus, according to LaSalle's Invariance Principle, we can obtain that the trajectory $\theta$ starting from any arbitrary initial point in $CD$ will converge to the largest invariant subset $RM$ of $\Omega$. $RM$ is the subset of $\Omega$ without local maxima and saddle points. 

\noindent Step 5: Use stochastic approximation theory.
We have applied Lyapunov function and LaSalle's Invariance Principle to prove that the model can converge to the largest invariant set $M$ of $\Omega$ in the entire parameter space. We will employ stochastic approximation theory to ensure the model's stable convergence in the iterative algorithm with stochastic noise. To ensure the stable convergence, we establish three assumptions.

\begin{itemize}
\item Step-size assumption: learning rate $\eta_{t}$ at the $t$th iteration should satisfy the Robbins-Monro condition, expressed as $\sum \eta_{t} = \infty$ and $\sum \eta_{t}^2 \leq \infty$. We adopt a decayed learning rate such as $\eta_{t} = \eta_{0}/\sqrt{t}$ to satisfy this condition.
\item Bounded-variance assumption: the noise term in the gradient estimate is required to have bounded variance. $L_{\text{pruning}}$ suppresses poorly performing strategies, $L_{\text{PG}}$ regularizes each agent's policy update, and reward normalization further stabilizes the gradient estimate.
\item Lipschitz-continuity assumption: the gradient is required to be Lipschitz continuous in the local optimization region.
\end{itemize}

According to the Robbins-Monro theorem, if the aforementioned assumptions are satisfied, the update process for the model parameters $\theta$ will stably converge to a local minimizer $\theta^{*}$, at which the gradient is zero.

\noindent Step 6: Connection to local equilibrium. As argued above, the competition can converge to a point with zero gradient. This step analyzes whether the point supports a local equilibrium interpretation. Assume that at a minimizer point, an agent $i$ can obtain higher reward if it changes its strategy through a local update of $\theta^{*}_{i}$. This leads to two cases:

(1) The $L_{total}$ decreases further. This situation cannot occur when $\theta^{*}$ is at a local minimum, as it contradicts the fact that the current gradient of $L_{total}$ is zero.

(2) $L_{total}$ instead increases. This event is likely caused by $L_{diversity}$: the improvement in the predictive capacity of agent $i$’ could be the result of its strategy converging with those of high-performing agents. This convergence, in turn, leads to an increase in $L_{diversity}$, which subsequently causes $L_{total}$ to increase. In this scenario, the model will continue to use gradient descent to guide each agent's $\theta^{*}_{i}$ back to a local minimum. 

In summary, at a local minimum, no agent can further improve its payoff by a local unilateral strategy update. Therefore, the local minimum can be interpreted as a local equilibrium of the induced optimization game.
\end{proof}

\subsection{Proof of Bayesian Regret}
\label{sec:regret_proof}

In this section, we prove that our framework is designed to achieve sublinear total regret, i.e., $RE_T = O(\sqrt{T})$. This indicates that the agents' strategies can improve faster than under linear regret, moving toward a stable local equilibrium rather than accumulating persistent suboptimality. Existing studies have also proved that in the non-convex landscape of deep learning, a well-designed optimization scheme can ensure that the regret, and hence the associated theoretical error, vanishes asymptotically, albeit typically converging to a local optimum \cite{OPT-013}. 

\begin{define}[Bayesian Regret]
As described in Equation~(\ref{eq:agent_state}), the state $o_{i}^{(t)}$ for agent $i$ in the $t$th round is $\{X_s,\mathcal{C}_s,\pi_{i}^{(t-1)},\pi_{-i}^{(t-1)}\}$. The Bayesian regret for agent $i$ over $T$ steps is defined as
\begin{equation}
    RE_i(T) = \mathbb{E} \left[ \sum_{t=1}^T (V_i^*(o^{t}_{i}) - V_i^{\pi_{i}^{(t)}}(o^{t}_{i})) \right]
\end{equation}
where $V_i^*$ is the optimal value function under BNE and $V_i^{\pi_t}$ is the value function under the policy $\pi_i^{(t)}$ at step $t$. The value function $V_i$ is defined as the long-term expected reward of agent $i$ based on its state $o_i$. It can be represented as

\vspace{-0.5cm}
\begin{equation}
V_{i}^{\pi_{i}^{(t)}} = \mathbb{E}_{\pi} [\sum_{k=t+1}^{\infty} \gamma^{k}\times r_i(o^{(k)}_i,a^{(k)}_i)]
\end{equation}
where $\gamma \in [0,1]$ is the discount factor for expected rewards. $a^{(k)}_i$ is the action of agent $i$ in $k$th round. $r_i$ is the reward based on $o^{(k)}_{i}$ and $a^{(k)}_i$. According to the Bellman optimality equation, the $Q$ function, which estimates future expected rewards based on current states and actions, is defined as

\vspace{-0.5cm}
\begin{equation}
\begin{aligned}
Q^{\pi}_{i}(o,a_i,a_{-i}) = & r_i(o,a_i,a_{-i}) + \\
&\gamma \times \sum_{ox}P(ox|o,a_i) \times V_i^{\pi}(ox)
\end{aligned}
\label{eq:q_function_def}
\end{equation}

\noindent where $a_{-i}$ are the actions of all agents except agent $i$, $o$ is the state of agent $i$ in the $t$th round, $ox$ is all the possible states of agent $i$ in the $(t+1)$th round, and $P$ is the probability of $ox$ under the condition of $o$ and $a_i$. To calculate Bayesian Regret, we first define and derive the performance difference between two different strategies $\pi^{'}$ and $\pi$.      
\end{define}

\begin{lemma}[Performance Difference]
\label{lemma:performance_diff}
For joint policies $\pi$ and $\pi'$, the difference in value is:
\begin{equation}
\begin{aligned}
V_i^{\pi'}(o) - V_i^\pi(o)
= & \frac{1}{1-\gamma}
\mathbb{E}_{o \sim d_{\pi'}}[\mathbb{E}_{a \sim \pi'} Q_i^\pi(o, a) - \\
& \mathbb{E}_{a \sim \pi} Q_i^\pi(o, a) ]
\end{aligned}
\end{equation}
\end{lemma}

\begin{proof}[Proof]
First, we expand $V_i^{\pi^{'}}(o)$ by time step $t$, the formula can be seen as below:

\begin{equation}
\begin{aligned}
V_{i}^{\pi^{'}}(o)
&= \mathbb{E}_{\pi^{'}}[r_i(o^{(0)},a^{(0)})] \\
&\quad + \sum_{t=1}^{\infty}\gamma^{t}\mathbb{E}_{\pi^{'}}[r_i(o^{(t)},a^{(t)})]
\end{aligned}
\label{eq:v_expansion}
\end{equation}

We set $\pi$ as the baseline policy and $\pi^{'}$ as the improved policy. Then, according to Equation~(\ref{eq:q_function_def}), we can derive $r_i(o^{(0)},a^{(0)}) = Q^{\pi_{i}}(o^{(0)},a^{(0)}) - \gamma \times \sum_{ox_1}P(ox_1|o^{(0)},a^{(0)}_i) \times V_i^{\pi}= Q^{\pi_{i}}(o^{(0)},a^{(0)}) - \gamma \times \mathbb{E}_{ox}[V_i^{\pi}(ox^{(1)}]$. Substituting this into Equation~(\ref{eq:v_expansion}), we obtain the following expression: 

\begin{equation}
\begin{aligned}
V^{\pi^{'}}_{i} = & \mathbb{E}_{\pi^{'}}[Q^{\pi}_{i}(o^{(0)}, a^{(0)}) - \gamma \times \mathbb{E}_{ox}[V_i^{\pi}(ox)]] + \\ 
& \sum_{t=1}^{\infty}\gamma^{t}\mathbb{E}_{\pi^{'}}[r_i(o^{(t)},a^{(t)})] \\
= & \mathbb{E}_{\pi^{'}}[Q^{\pi}_{i}(o^{(0)}, a^{(0)})] - \gamma \times \mathbb{E}_{\pi^{'}}[V_i^{\pi}(ox^{(1)})] + \\
& \sum_{t=1}^{\infty}\gamma^{t}\mathbb{E}_{\pi^{'}}[r_i(o^{(t)},a^{(t)})]
\end{aligned}
\label{eq:v_substitution}
\end{equation}

We can derive the third term $\sum_{t=1}^{\infty}\gamma^{t}\mathbb{E}_{\pi^{'}}[r_i(o^{(t)},a^{(t)})]$ from formula (28) as $\gamma \times \mathbb{E}_{\pi^{'}}[V_{i}^{\pi^{'}}(ox^{(1)})]$. Substituting the third term into Equation~(\ref{eq:v_substitution}) gives the expression $V^{\pi^{'}}_{i}$.

\vspace{-0.5cm}
\begin{equation}
\begin{aligned}
V^{\pi^{'}}_{i}(o) = & \mathbb{E}_{\pi^{'}}[Q^{\pi}](o^{0},a^{(0)}) - \\
& \gamma \times \mathbb{E}_{\pi^{'}}[V^{\pi}_{i}(o^{(1)})] + \\
& \gamma \times \mathbb{E}_{\pi^{'}}[V^{\pi^{'}}_{i}(o^{(1)})]
\end{aligned}
\end{equation}

Based on the expression of $V^{\pi^{'}}_{i}(o)$, we perform a further derivation of $V^{\pi^{'}}_{i} - V^{\pi}_{i}$ as follows:

\begin{equation}
\begin{aligned}
V^{\pi^{'}}_{i}(o) - V^{\pi}_{i}(o) = & \mathbb{E}_{\pi^{'}}[Q^{\pi}_{i}(o^{(0)},a^{(0)})] + \\
& \gamma \mathbb{E}_{\pi^{'}}[V^{\pi^{'}}_{i}(o^{(1)}) - V^{\pi}_{i}(o^{(1)})] \\
& - V^{\pi}_{i}(o) \\
= & \mathbb{E}_{\pi^{'}}[Q^{\pi}_{i}(o^{(0)},a^{(0)}) - V^{\pi}_{i}(o^{(0)})] \\
& + \gamma \mathbb{E}_{\pi^{'}}[V^{\pi^{'}}_{i}(o^{(1)}) - V^{\pi}_{i}(o^{(1)})]
\end{aligned}
\end{equation}

Let
\[
A_t = Q^{\pi}_{i}(o^{(t)},a^{(t)}) - V^{\pi}_{i}(o^{(t)}).
\]
Recursively expanding the second term on the right-hand side of (30) analogously yields:

\begin{equation}
\begin{aligned}
&V^{\pi^{'}}_{i}(o^{(1)}) - V^{\pi}_{i}(o^{(1)}) \\
&\quad = \mathbb{E}_{\pi^{'}}[A_1]
+ \gamma \mathbb{E}_{\pi^{'}}\!\left[
V^{\pi^{'}}_{i}(o^{(2)}) - V^{\pi}_{i}(o^{(2)})
\right].
\end{aligned}
\end{equation}

Repeating this process yields:

{\small
\begin{equation}
\begin{aligned}
&V^{\pi^{'}}_{i}(o) - V^{\pi}_{i}(o) \\
&\quad = \mathbb{E}_{\pi^{'}}[A_0]
+ \gamma \mathbb{E}_{\pi^{'}}[A_1]
+ \gamma^2 \mathbb{E}_{\pi^{'}}[A_2] + \cdots \\
&\quad = \sum_{t=0}^{\infty} \gamma^{t}\mathbb{E}_{\pi^{'}}[A_t] \\
&\quad =
\sum_{t=0}^{\infty} \gamma^{t}
\mathbb{E}_{o^{(t)} \sim d^{t}_{\pi^{'}}}
\!\left[
Q^{\pi}_{i}(o^{(t)},a^{(t)}) - V^{\pi}_{i}(o^{(t)})
\right].
\end{aligned}
\end{equation}
}

\noindent where $o^{(t)} \sim d^{t}_{\pi^{'}}$ represents the state distribution in $t$th round under policy $\pi^{'}$. According to the stationarity of the Markov chain, the state distribution is identical across all time steps. Assume the distribution is $o \sim d_{\pi^{'}}$, formula (32) can be transformed into the following format:

\begin{equation}
\begin{aligned}
& \sum_{t=0}^{\infty} \gamma^{t} \times \mathbb{E}_{o^{(t)} \sim d^{t}_{\pi^{'}}} [Q^{\pi}_{i}(o^{(t)},a^{(t)}) - V^{\pi}_{i}(o^{(t)})] = \\
& \sum_{t=0}^{\infty} \gamma^{t} \times \mathbb{E}_{o \sim d_{\pi^{'}}}[Q^{\pi}_{i}(o,a) - V^{\pi}_{i}(o)]
\end{aligned}
\end{equation}

Since $o \sim d_{\pi'}$ is independent of $t$, and according to classical
reinforcement learning theory, we have
\begin{equation}
\begin{aligned}
V_i^{\pi}(o)
&= \sum_{t=0}^{\infty} V_i^{\pi}\!\left(o^{(t)}\right) \\
&= \sum_{t=0}^{\infty}
\mathbb{E}_{a^{(t)} \sim \pi(\cdot \mid o^{(t)})}
\!\left[ Q_i^{\pi}\!\left(o^{(t)}, a^{(t)}\right) \right] \\
&= \mathbb{E}_{a \sim \pi(\cdot \mid o)}\!\left[ Q_i^{\pi}(o,a) \right].
\end{aligned}
\end{equation}
The equation can be rewritten as

\begin{equation}
\begin{aligned}
V_i^{\pi'}(o) - V_i^\pi(o)
&= \frac{1}{1-\gamma} \times \mathbb{E}_{o \sim d_{\pi'}} \Bigl[
\mathbb{E}_{a \sim \pi'} [Q_i^\pi(o, a)] \\
&\quad - \mathbb{E}_{a \sim \pi} [Q_i^\pi(o, a)]\Bigr]
\end{aligned}
\end{equation}

Lemma~\ref{lemma:performance_diff} is proved. Building on the derivation of the regret formula (35), we will analyze the convergence bounding of Bayesian Regret.
\end{proof}

\begin{theorem}[Bounding of Bayesian Regret]
\label{thm:regret_bound}
The Bayesian Regret of EGT-based competition will converge to a sublinear regret bound $O(\sqrt{T})$ when time step $T$ approaches to $\infty$. 
\end{theorem}

\begin{proof}[Proof]
We decompose the Q-value difference into an estimation error term $\epsilon_t(o,a)$ and a policy suboptimality term $\Delta_t(o,a)$
\begin{equation}
\begin{aligned}
Q_i^{\pi}(o^{(t)*}, a^{(t)*}) - Q_i^{\pi}(o^{(t)}, a^{(t)})
&= \epsilon_t(o,a) + \Delta_t(o,a)
\end{aligned}
\end{equation}

\noindent where $\epsilon_t(o,a) = Q^{\pi*}(o^{(t)},a^{(t)}) - Q^{\pi}(o^{(t)},a^{(t)})$ and $\Delta_t(o,a) = Q_i^{\pi}(o^{(t)*}, a^{(t)*}) - Q^{\pi*}(o^{(t)},a^{(t)})$. For the estimation error $\epsilon_t(o,a)$, based on the theory of stochastic approximation, we can obtain the following:

\begin{equation}
\mathbb{E}[\epsilon_{t+1}] \leq (1 - \eta_{t}(1 - \gamma))\mathbb{E}[\epsilon_t] + \eta_{t}^{2}\sigma^{2}
\end{equation}

\noindent where $\sigma$ is the variance, $\eta_t = \eta_0/\sqrt{t}$ is the learning rate of $t$th iteration. Substituting $\eta$ into formula (37), we obtain

\begin{equation}
\mathbb{E}[\epsilon_{t+1}] \leq (1-\frac{\epsilon_{0}(1-\gamma)}{\sqrt{t}})\mathbb{E}[\epsilon_t] + \frac{\eta^2_{0}\sigma^{2}}{t}
\end{equation}

Let $a_t = \epsilon_{0}(1-\gamma)/\sqrt{t}$ and $b_t=\eta^2_{0}\sigma^{2}/{t}$. We can derive that

\begin{equation}
\begin{aligned}
\mathbb{E}[\epsilon_{t+1}] \leq & a_t \times \mathbb{E}[\epsilon_t] + b_t \\
= & a_{t}a_{t-1}\mathbb{E}[\epsilon_{t-1}] + a_{t}b_{t-1} + b_{t} \\
= & ...... \\
= & (\prod_{k=1}^{t}a_{k})\mathbb{E}[\epsilon_1] + \sum_{j=1}^{t}(\prod_{k=j+1}^{t}a_{k})b_{j}
\end{aligned}
\end{equation}

For the product term $\prod_{k=1}^{t}a_{k}$, we apply a logarithmic approximation $ln(1-x) \approx -x$  to obtain the following:

\begin{equation}
\begin{aligned}
ln(\prod_{k=1}^{t}a_{k})= & \sum_{k=1}^{t}ln(1-\frac{\eta_{0}(1 - \gamma)}{\sqrt{k}}) \\ \approx & -\eta_{0}(1-\gamma)\sum_{k=1}^{t}\frac{1}{\sqrt{k}} \\
\approx & -2\eta_{0}(1-\gamma)\sqrt{t}
\end{aligned}
\end{equation}

\noindent where $\sum_{k=1}^{t}\frac{1}{\sqrt{k}} \approx 2\sqrt{t}$ is based on differential approximation. $\prod_{k=1}^{t}a_{k}$ decays exponentially to 0 as $t$ increases.

For $\sum_{j=1}^{t}(\prod_{k=j+1}^{t}a_{k})b_{j}$, using a similar strategy of logarithmic and integral approximations, we derive

\begin{equation}
\begin{aligned}
\sum_{j=1}^{t}(\prod_{k=j+1}^{t}a_{k})b_{j} \approx &  
\sum_{j=1}^{t}\frac{\eta_{0}^{2}\sigma^{2}}{t}(\frac{j}{t})^{(\eta_{0}(1-\gamma))^{2}/2} \\
& \times e^{2\eta_{0}(1-\gamma)(\sqrt{t}-\sqrt{j})}
\end{aligned}
\end{equation}

According to asymptotic analysis theory, when $j = t - m$, $m = c \times \sqrt{t}$ which can also be represented as $m = O(\sqrt{t})$ and this indicates that $m$ is smaller than $t$ when $t$ is very large. Because t and j are very close in value, the exponential term can be approximated by 1, allowing Equation (41) to be rewritten as

\begin{equation}
\begin{aligned}
\sum_{j=1}^{t}(\prod_{k=j+1}^{t}a_{k})b_{j} \approx \sum_{k=0}^{O(\sqrt{t})}\frac{\eta_{0}^{2}\sigma^{2}}{t-c\sqrt{t}}O(1)
\end{aligned}
\end{equation}

For $O(\sqrt{\sqrt{t}})$ summation terms, each with a magnitude of $O(1/t)$, equation (42) can be rewritten as

\begin{equation}
\sum_{j=1}^{t}(\prod_{k=j+1}^{t}a_{k})b_{j} \approx O(\sqrt{t})O(1/t) = O(1/\sqrt{t})
\end{equation}

This proof shows that the estimated error $\epsilon_t(o,a)$ of the competition framework converges at a rate of $1/\sqrt{t}$.

For the policy suboptimality term $\Delta_t(o,a)$, according to the research of online convex optimization, we obtain
{\small
\begin{equation}
\begin{aligned}
\Delta_t(o,a) = & Q^{\pi}_{i}(o^{(t)*},a^{(t)*})-Q^{\pi^{*}}_{i}(o^{(t)},a^{(t)}) \\
& \leq \frac{1}{t}\sum_{\tau=1}^{t}[Q^{\pi}_{i}(o^{(t)*},a^{(t)*})-Q^{\pi^{*}}_{i}(o^{(\tau)},a^{(\tau)})] \\
& \leq \frac{D^{2}(\mathbb{P})}{2\eta_{0}\sqrt{t}} + \frac{\eta_0L^2}{2\sqrt{t}} \approx O(\frac{1}{\sqrt{t}})
\end{aligned}
\end{equation}
}
This proof shows that the suboptimality term $\Delta_t(s,a)$ of the competition framework also converges at a rate of $1/\sqrt{t}$.

The total regret $RE(T) = \sum_{i=1}^N RE_i(T)$ is therefore bounded by:
{\small
\begin{equation}
    RE(T) \leq \sum_{i=1}^N \sum_{t=1}^T \frac{1}{1-\gamma} \left( \frac{1}{\sqrt{t}} + \frac{1}{\sqrt{t}} \right) = \frac{N}{1-\gamma} \sum_{t=1}^T \frac{1}{\sqrt{t}}
\end{equation}
}
Since $\sum_{t=1}^T 1/\sqrt{t} \leq 2\sqrt{T}$, we arrive at the final bound:
\begin{equation}
    RE(T) = O\left(\frac{N\sqrt{T}}{1-\gamma}\right)
\end{equation}

This analysis provides theoretical support that the evolutionary dynamics of CompEvo can promote stable strategy improvement under the stated assumptions, with the potential to achieve a sublinear regret bound.
\end{proof}

\subsection{Comparative Analysis: Linear Regret in Adversarial Frameworks}
\label{sec:mad_comparison}

To highlight the theoretical advantages of our EGT-inspired approach, we provide a comparative analysis of a purely adversarial MAD framework. Such frameworks are a common alternative for complex reasoning but, as we will show, are fundamentally limited. In tasks like forecasting, a MAD setup can be modeled as a zero-sum game, where one agent's gain (e.g., its proposed logic being accepted through debate) is exactly another agent's loss. We prove that such a system suffers from linear total regret, i.e., $RE_{\text{debate}}(T) = O(T)$.

\begin{lemma}[Persistent Suboptimality in Zero-Sum Games]
\label{lemma:linear_regret}
In a two-player zero-sum game that lacks a pure-strategy BNE, any agent $i$ employing a no-regret learning algorithm will exhibit persistent suboptimality. That is, there exists a constant $\delta_{\min} > 0$ such that for all $t$ beyond some initial phase $T_0$:
\begin{equation}
    \mathbb{E}\left[\max_{a_i^*} Q_i^*(o^{(t)}, a^{(t)*}) - V_i^{\pi}(o^{(t)})\right] \ge \delta_{\min}
\end{equation}
where $V_i^{\pi}(o^{(t)}) = \mathbb{E}_{a_i^t \sim \pi(o^{(t)})}[Q_i^*(o^{(t)}, a_{i}^{(t)})]$ is the value of the agent's policy at step $t$.
\end{lemma}
\begin{proof}
According to the mini-max theorem, the equilibrium in such games necessitates mixed strategies. Intuitively, an agent cannot adopt a deterministic policy, as it would be quickly predicted and exploited by a rational opponent. To remain unpredictable and avoid exploitation, the agent's policy $\pi_i^t$ must maintain a minimum level of entropy, i.e., $H(\pi_i^t) \ge h_{\min} > 0$. This required randomness implies that the agent cannot exclusively play the optimal action, leading to a persistent, non-zero gap between its policy's expected value and the maximum possible value. This performance gap is lower-bounded by a constant $\delta_{\min} > 0$.
\end{proof}

The total regret for the MAD framework, $RE_{\text{debate}}(T)$, is the sum of these persistent single-step regrets. Its lower bound can be derived as
\begin{equation}
\begin{aligned}
    RE_{\text{debate}}(T) & = \sum_{i=1}^N \sum_{t=1}^T \mathbb{E}[\Delta_t] \\
    & \ge \sum_{i=1}^N \sum_{t=T_0}^T \delta_{\min} \\
    & = N(T-T_0)\delta_{\min} = O(NT)
\end{aligned}
\end{equation}
This linear regret bound starkly contrasts with the sublinear regret achieved by our EGT-inspired framework, as theoretically justified in Appendix~\ref{sec:theory}. It provides a formal justification for why zero-sum dynamics are ill-suited for achieving stable, long-term strategy improvement. In such systems, agents are perpetually trapped in a cycle of counter-exploitation, preventing them from converging to a globally effective strategy.

\subsection{Aggregation Mechanism: Monotonicity of Fitness-based Weighting}
\label{sec:monotonicity_proof}

In our framework, the learning and evolution of agents rely on a shared global reward signal derived from the aggregated prediction. To ensure that an improvement in an individual agent's strategy reliably translates to an improvement in global performance, it is crucial to establish that our aggregation mechanism is monotonic. This means that an enhancement in any single agent's predictive performance, reflected by a higher long-term fitness, should lead to a non-detrimental, and typically positive, impact on the final aggregated outcome. This is crucial for proving the effectiveness of our competition, especially the survival-of-the-fittest mechanism. 

We consider an agent's long-term fitness $M_i$ in Equation~(\ref{eq:fitness}) as a proxy for its local value. The final aggregated prediction $\hat{\mathbf{Y}}_{\text{agg}}$ is defined by the weighting rule in Equation~(\ref{eq:weight_softmax}), and its performance is measured by the prediction loss $L_{\text{predict}}$. Our goal is to prove that as an agent's long-term fitness increases, its contribution weight to the final prediction is non-decreasing. This ensures that better-performing agents are rightfully given more influence.

\begin{prop}[Monotonicity of Aggregation Weights]
\label{prop:monotonicity}
In our framework, if the long-term fitness $M_k$ of an agent $k$ increases, its weight $w_k$ in the final aggregated prediction is non-decreasing, i.e., $\frac{\partial w_k}{\partial M_k} \ge 0$.
\end{prop}

\begin{proof}
Recall the contribution weight $w_k$ defined in Equation~(\ref{eq:weight_softmax}) in the main text:
\[
w_k = \frac{\exp((g_k \cdot M_k) / \tau)}{\sum_{j=1}^{I} \exp((g_j \cdot M_j) / \tau)}
\]

To analyze its monotonicity with respect to $M_k$, we compute the partial derivative $\frac{\partial w_k}{\partial M_k}$. Let $z_j = (g_j \cdot M_j) / \tau$ for simplification. The expression becomes $w_k = \frac{e^{z_k}}{\sum_j e^{z_j}}$.

Using the quotient rule for differentiation with respect to $z_k$:
\begin{align}
    \frac{\partial w_k}{\partial z_k} &= \frac{\frac{\partial(e^{z_k})}{\partial z_k} \left(\sum_j e^{z_j}\right) - e^{z_k} \frac{\partial(\sum_j e^{z_j})}{\partial z_k}}{\left(\sum_j e^{z_j}\right)^2} \nonumber \\
    &= \frac{e^{z_k} \left(\sum_j e^{z_j}\right) - e^{z_k} (e^{z_k})}{\left(\sum_j e^{z_j}\right)^2} \nonumber \\
    &= \frac{e^{z_k} \left(\sum_{j \neq k} e^{z_j}\right)}{\left(\sum_j e^{z_j}\right)^2}
\end{align}
Since the exponential function yields a positive value, and assuming the gating parameter $g_j \ge 0$ and temperature $\tau > 0$, all terms in the numerator and denominator are non-negative. Thus, $\frac{\partial w_k}{\partial z_k} \ge 0$. According to the chain rule, we have:
\begin{equation}
    \frac{\partial w_k}{\partial M_k} = \frac{\partial w_k}{\partial z_k} \cdot \frac{\partial z_k}{\partial M_k} = \frac{\partial w_k}{\partial z_k} \cdot \frac{g_k}{\tau}
\end{equation}
Given that $g_k \ge 0$ and $\tau > 0$, it follows that $\frac{\partial w_k}{\partial M_k} \ge 0$.

Conclusion: this proof demonstrates that as an agent's long-term fitness ($M_k$) improves, its contribution weight ($w_k$) to the final group decision non-decreasingly increases. 
\end{proof}

\section{Additional Framework Details}
\label{sec:framework_details}
\subsection{Soft Pruning}
\label{SF}

Soft pruning down-weights persistently weak agents without discrete elimination. It is implemented through the learnable gate used in the fitness-based weighting rule and an L1 penalty on gate magnitudes.

We first compute each agent's long-term fitness $M_i$ as the EMA of its per-round reward $R_i$ (see Eq.~(\ref{eq:reward})--(\ref{eq:fitness}) in the main text), and determine its contribution weight $w_i$ through the gated softmax selection (Eq.~(\ref{eq:weight_softmax})).
To enable soft-pruning, we impose an L1 penalty on the learnable gates $\{g_i\}_{i=1}^{N}$:
\begin{equation}
\mathcal{L}_{\text{pruning}} = \sum_{i=1}^{N} |g_i|.
\label{eq:pruning_loss_app}
\end{equation}
When an agent has low long-term fitness, the gate penalty reduces its contribution weight in Eq.~(\ref{eq:weight_softmax}). The agent is therefore suppressed smoothly during training, while still remaining available for later recovery if its strategy improves.

\subsection{Logic Diversity}
\label{IA}

The diversity term acts on the candidate logic representation generated by each agent before opponent-aware mixing. For agent $i$, this representation is denoted as $h_i^c$. We penalize pairwise cosine similarity among candidate logics:
\begin{equation}
\mathcal{L}_{\text{diversity}} = \sum_{i=1}^{N} \sum_{j=i+1}^{N} \frac{h_i^c \cdot h_j^c}{\|h_i^c\|_2 \|h_j^c\|_2} \label{eq:diversity_loss_app}
\end{equation}
Minimizing this term encourages agents to maintain distinct candidate logics before they exchange opponent information. This is the implementation detail behind the diversity regularization used in the main text.

\subsection{Module Interaction Summary}
\label{subsec:diversity_discussion}

The trainable loop has three interacting terms:
\begin{itemize}
    \item $L_{\text{pg}}$ links forecasting rewards to logic-generation updates.
    \item $L_{\text{div}}$ keeps candidate logics separated before opponent-aware fusion.
    \item $L_{\text{prune}}$ and the gated softmax in Eq.~(\ref{eq:weight_softmax}) reduce the influence of weak long-term strategies.
\end{itemize}
The first term updates strategies, the second preserves variation, and the third implements differentiable selection.

\section{Experimental Setup Details}
\label{sec:experimental_settings}

\subsection{Details of Datasets}
\label{dataset}
The dataset sources, scales, and news filtering examples are provided for reproduction.

\begin{table*}[htbp]
\centering
\small
\setlength{\tabcolsep}{4pt}
\renewcommand{\arraystretch}{1.12}
\begin{tabular}{@{}p{0.11\textwidth}p{0.12\textwidth}p{0.31\textwidth}p{0.07\textwidth}p{0.09\textwidth}p{0.22\textwidth}@{}}
\toprule
Dataset & Domain & Time-series source & Variables & Time points & News source \\
\midrule
Electricity & Energy & AEMO operational demand~\cite{aemo_operational_demand} & 19 & 52,560 & news.com.au/GDELT~\cite{leetaru2013gdelt} \\
Exchange & Finance & Standard exchange-rate series~\cite{lai2018modeling} & 7 & 1,460 & news.com.au/GDELT~\cite{leetaru2013gdelt} \\
Bitcoin & Cryptocurrency & BitInfoCharts-derived series~\cite{bitinfocharts} & 18 & 858 & GDELT/domain crawls~\cite{leetaru2013gdelt} \\
Traffic & Transportation & Caltrans PeMS~\cite{caltrans_pems} & 862 & 17,544 & GDELT/domain crawls~\cite{leetaru2013gdelt} \\
\bottomrule
\end{tabular}
\caption{Dataset provenance and time-series scale.}
\label{tab:dataset_provenance}
\end{table*}

Table~\ref{tab:dataset_provenance} summarizes dataset provenance and scale.
Representative news examples are listed below to illustrate the type of textual evidence aligned with each dataset.

\begin{table*}[htbp]
\centering
\small
\setlength{\tabcolsep}{5pt}
\renewcommand{\arraystretch}{1.12}
\begin{tabular}{p{0.16\textwidth}p{0.76\textwidth}}
\toprule
Dataset & Representative news example \\
\midrule
Electricity & South Australia is only days away from a heatwave which will last for almost a week and has left Tour Down Under organisers anxiously watching the weather forecast. \\
Exchange & The RBA has dramatically revised down its economic forecasts amid the ongoing property market correction, prompting the Australian dollar to plunge again. \\
Traffic & A funnel cloud was spotted over Waterford in northern California on April 27 as a line of storms brought heavy rain and hail to the area. \\
Bitcoin & Personal finance expert Peter Adeney, known as `Mr. Money Mustache,' has warned against investing in bitcoin, calling it a speculative asset rather than a true investment. \\
\bottomrule
\end{tabular}
\caption{Representative news examples for each dataset.}
\label{tab:dataset_news}
\end{table*}

\subsection{Implementation Details}
\label{sec:implementation_details}

\paragraph{Base model and efficiency design.}
\label{sec:base_model}
We use 10 agents with 5-fold validation under the same train/validation/test protocol for all methods. Agents share an instruction-tuned backbone, while agent-specific lightweight adapters maintain heterogeneous reasoning behavior. The backbone parameters are kept frozen during agent-specific updates.

Each agent uses separate LoRA adapters for forecasting and logic generation, both with rank 16 and scaling factor $\alpha_{\mathrm{LoRA}}=32$. Candidate textual logic is tokenized with the shared Llama-3.1-8B tokenizer and encoded using the corresponding agent's logic adapter; the last-layer hidden state at the final non-padding token forms the 4,096-dimensional logic representation. For agent $i$, scalar interaction scores over peer agents are normalized row-wise to obtain $\alpha_{ij}$, and the opponent context is the weighted sum of peer logic representations. The two 4,096-dimensional gates in the three-stage update then fuse the agent and opponent representations elementwise. The fused representation is linearly projected into one soft-prompt embedding, prepended to the logic-generation instruction, and decoded autoregressively by the shared Llama-3.1-8B backbone into the next textual logic.

\paragraph{Policy optimization.}
\label{GRPO} 

We use group relative policy optimization (GRPO) \cite{deepseek_math_grpo} to align logic generation with forecasting rewards. The update uses group-normalized rewards, a KL penalty against a frozen reference policy, AdamW optimization, 3\% warmup, and cosine learning-rate decay. The main hyperparameters are summarized in Table~\ref{tab:hyperparameters}.

\paragraph{Hyperparameters.}
The supervised forecasting stage uses cosine decay, and the input context window is truncated to 4096 tokens.

\begin{table}[htbp]
\centering
\small
\setlength{\tabcolsep}{5pt}
\renewcommand{\arraystretch}{1.05}
\begin{tabular}{lc}
\toprule
Hyperparameter & Value \\
\midrule
Number of agents & 10 \\
Validation folds & 5 \\
Adapter rank & 16 \\
Adapter alpha & 32 \\
Adapter dropout & 0.05 \\
GRPO group size & 8 \\
KL coefficient & 0.04 \\
Policy learning rate & $1\times10^{-6}$ \\
Forecasting learning rate & $1\times10^{-4}$ \\
PPO clipping range & 0.2 \\
Warmup ratio & 3\% \\
Context length & 4096 \\
\bottomrule
\end{tabular}
\caption{Main implementation hyperparameters.}
\label{tab:hyperparameters}
\end{table}

\subsection{Forecasting Input Template for Fine-tuning LLM}
\label{input}

The input-output construction follows \citet{Wang2024}. A shortened example is shown in the card below.

\begin{promptcard}{Forecasting Input Template for Fine-tuning}
\textbf{Instruction.}
Historical load data: 4640.1, 4476.7, ...

\vspace{0.4em}
\textbf{Input.}
Predict next-day load for VIC. Context includes date, holiday status, weather, and aligned news with rationales.

\vspace{0.4em}
\textbf{Output.}
4741.8, 4497.8, 4360.1, 4188.0, ...

\vspace{0.4em}
\textbf{Serialized format.}
\begin{tcolorbox}[
    colback=black!3,
    colframe=black!20,
    boxrule=0.4pt,
    arc=1pt,
    left=4pt,
    right=4pt,
    top=3pt,
    bottom=3pt
]
\begin{lstlisting}[basicstyle=\ttfamily\footnotesize,breaklines=true]
{
  "instruction": "Historical load data: 4640.1, 4476.7, ...",
  "input": "Predict next-day load for VIC. Context includes date, holiday status, weather, and aligned news with rationales.",
  "output": "4741.8, 4497.8, 4360.1, 4188.0, ..."
}
\end{lstlisting}
\end{tcolorbox}
\end{promptcard}
\subsection{Baseline Details}
\label{sec:baselines}

The baselines are grouped into numerical-centric models, news-aware LLM models, and multi-agent systems. All methods use the same dataset split, direct 48-step forecasting horizon, evaluation metrics, and five temporally ordered validation folds; methods that consume news receive the same candidate news pool. We retain each multi-agent baseline's original interaction mechanism while adapting its task interface and output to this forecasting setting. Table~\ref{tab:baseline_configuration} reports the single-agent and multi-agent LLM configurations and their trainable parameter counts. GPT-4o is frozen and therefore excluded from the MAE-GPT count. These configurations are not capacity matched: CompEvo uses 26.7\% fewer trainable parameters than MAD-RL and 2.20$\times$ as many as MAEL or MAE-GPT.

\begin{table}[htbp]
\centering
\small
\setlength{\tabcolsep}{3.5pt}
\renewcommand{\arraystretch}{1.05}
\resizebox{\columnwidth}{!}{%
\begin{tabular}{lclrr}
\toprule
Method & Agents & Model(s) & Trainable params & Train steps \\
\midrule
SA & 1 & Llama-3.1-8B & 41.9M & 2,193 \\
MAEL & 10 & Llama-3.1-8B & 419.4M & 2,193 \\
MAE-GPT & 10 & GPT-4o + Llama-3.1-8B & 419.4M & 2,193 \\
MAD-RL & 10 & Llama-3.1-8B & 1,258.4M & 2,193 \\
CompEvo & 10 & Llama-3.1-8B & 922.8M & 2,193 \\
\bottomrule
\end{tabular}
}
\caption{Multi-agent baseline configurations. All methods use a context length of 4,096 tokens. Trainable counts exclude frozen backbone parameters and the frozen GPT-4o selector in MAE-GPT.}
\label{tab:baseline_configuration}
\end{table}

\paragraph{Numerical-centric baselines.}
These methods rely solely on historical time-series numerical data for forecasting.

\begin{itemize}
    \item DLinear \cite{zeng2023transformers}: A linear decomposition baseline that separates trend and remainder components for time-series forecasting.
    
    \item iTransformer \cite{liu2023itransformer}: An inverted Transformer that treats each variate as a token to model multivariate dependencies.

    \item FiLM \cite{zhou2022film}: A frequency-enhanced Legendre memory model for preserving historical information in long-term forecasting.

    \item Pyraformer \cite{liu2021pyraformer}: A pyramidal-attention forecasting model that captures multi-resolution temporal dependencies with linear complexity.

    \item FEDformer \cite{zhou2022fedformer}: A decomposition-based Transformer that uses frequency-domain enhancement for efficient long-term forecasting.

    \item GPT4TS \cite{zhou2024one}: A pre-trained language model adaptation framework for time-series forecasting.
\end{itemize}

\paragraph{News-aware LLM baselines.}
These methods explicitly incorporate textual news or event information alongside numerical data.

\begin{itemize}
    \item News2Forecast (N2F) \cite{Wang2024}: A news-aware LLM pipeline that retrieves aligned news and uses self-refinement for forecasting.
    
    \item GPT4MTS \cite{liu2025dpgpt4mts}: A dual-prompt multimodal framework that aligns textual news with numerical trends.
    
    \item LTP-LLM \cite{xiong2025beyond}: A multimodal framework that uses LLM-generated trend descriptions as auxiliary forecasting features.
\end{itemize}

\paragraph{Multi-agent baselines.}
To isolate the contribution of our evolutionary competitive mechanism, we compare against recent multi-agent architectures.

\begin{itemize}
    \item MAEL (Multi-Agent Ensemble Learning): A parallel-agent averaging baseline without interaction or evolution.
    
    \item MAE-GPT \cite{yuan-etal-2025-evoagent}: A prompt-mutation evolutionary baseline based on EvoAgent.
    
    \item MAD-RL \cite{park-etal-2025-maporl}: Agents first generate individual 48-step forecasts and supporting news rationales, then exchange them through the RL-enhanced debate protocol and revise their forecasts. The final prediction is a weighted average of the revised forecasts, using normalized final-round verifier scores as agent weights. Forecasting loss over the 48-step target window is used as the task reward.
\end{itemize}

\section{Additional Experimental Results}
\label{sec:llm_generalization}
\label{sec:LLMs_tests}

\subsection{Discussion on Memorization Risk}
\label{subsec:performance_source_discussion}

Because LLMs may encode historical event knowledge from pretraining, we avoid attributing the gains of CompEvo solely to pretrained event memory. In our experiments, all compared LLM-based methods use the same temporal data split and the same available news pool. The ablation results in the main text further show that removing the evolutionary components leads to clear performance degradation under the same backbone and data setting, suggesting that the observed gains mainly come from the proposed competitive adaptation mechanism.

To further examine memorization risk, we construct an additional 2026 Electricity inference set using the same data provenance as the original Electricity benchmark in Table~\ref{tab:dataset_provenance}, namely AEMO operational demand data aligned with GDELT news. The set contains 100 samples from February 19 to March 30, 2026, covering Australian regions such as TAS, QLD, NSW, VIC, and SA. Each sample contains a 48-step historical load sequence, candidate news, and a 48-step ground-truth future sequence. We directly run inference using the Llama-3.1-based CompEvo model trained in the main experiments, without any further parameter update.

\begin{table}[t]
\centering
\footnotesize
\begin{tabular}{lcc}
\toprule
Metric & Original test & 2026 inference \\
\midrule
MSE & 125440.00 & 111141.83 \\
RMSE & 354.18 & 333.38 \\
MAE & 220.15 & 225.02 \\
MAPE (\%) & 6.45 & 6.99 \\
\bottomrule
\end{tabular}
\caption{Leakage-controlled inference results on Electricity. The 2026 samples are used only for post-training inference and are not used for training, validation, model selection, or hyperparameter tuning.}
\label{tab:memorization_risk_2026}
\end{table}

As shown in Table~\ref{tab:memorization_risk_2026}, CompEvo maintains comparable performance on the 2026 samples, with an RMSE of 333.38 and a MAPE of 6.99\%. This directly suggests that the model does not rely only on memorized historical event--target associations. Instead, the trained model can still generalize to temporally later electricity data, supporting that CompEvo learns a transferable news-conditioned forecasting mechanism rather than merely exploiting pre-training exposure. We view this experiment as a leakage-controlled sanity check, while fully dynamic benchmarks with continuously updated news streams remain important for future evaluation.

\subsection{Robustness Analysis under News Perturbations}
\label{sec:noisy_news}
\vspace{-0.2cm}

Figure~\ref{fig:noisy_news_robustness} shows RMSE of CompEvo, MAEL, MAE-GPT, and MAD-RL on all four datasets under shuffled, irrelevant, noisy, and original news. CompEvo degrades the slowest, while MAEL and MAE-GPT show larger increases and MAD-RL clusters evidence-seeking patterns, overreacting to perturbations. This demonstrates that fitness-based selection and competitive evolution maintain diverse, specialized reasoning, with similar trends across MAE, MSE, and MAPE.

\begin{figure}[t]
  \centering
  \includegraphics[width=\columnwidth]{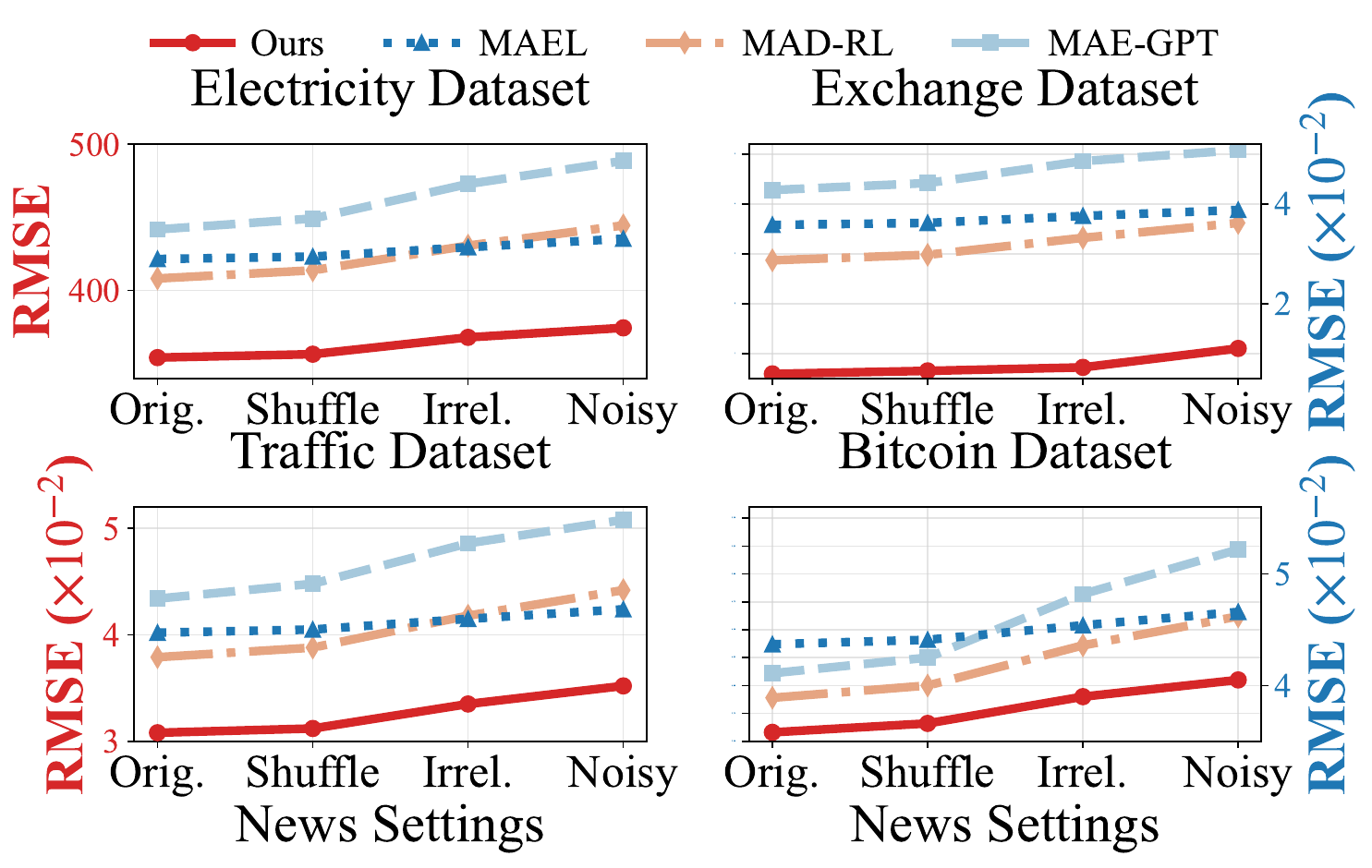}
\caption{RMSE under news perturbations.}
  \label{fig:noisy_news_robustness}
\end{figure}

\subsection{Additional Hyperparameter Sensitivity Analysis}
\label{sec:Parameter Sensitivity Analysis}
\label{sec:parameter_sensitivity}

Both diversity and pruning regularization work best at moderate strength. Figure~\ref{fig:sensitivity_analysis} shows the two representative sweeps discussed in the main experimental design. The remaining sensitivity analyses below further test population size, policy-update weight, generation temperature, aggregation temperature, prompt robustness, and epoch selection.

\begin{figure}[h]
  \centering
  \begin{subfigure}{\linewidth}
    \centering
    \includegraphics[width=\linewidth]{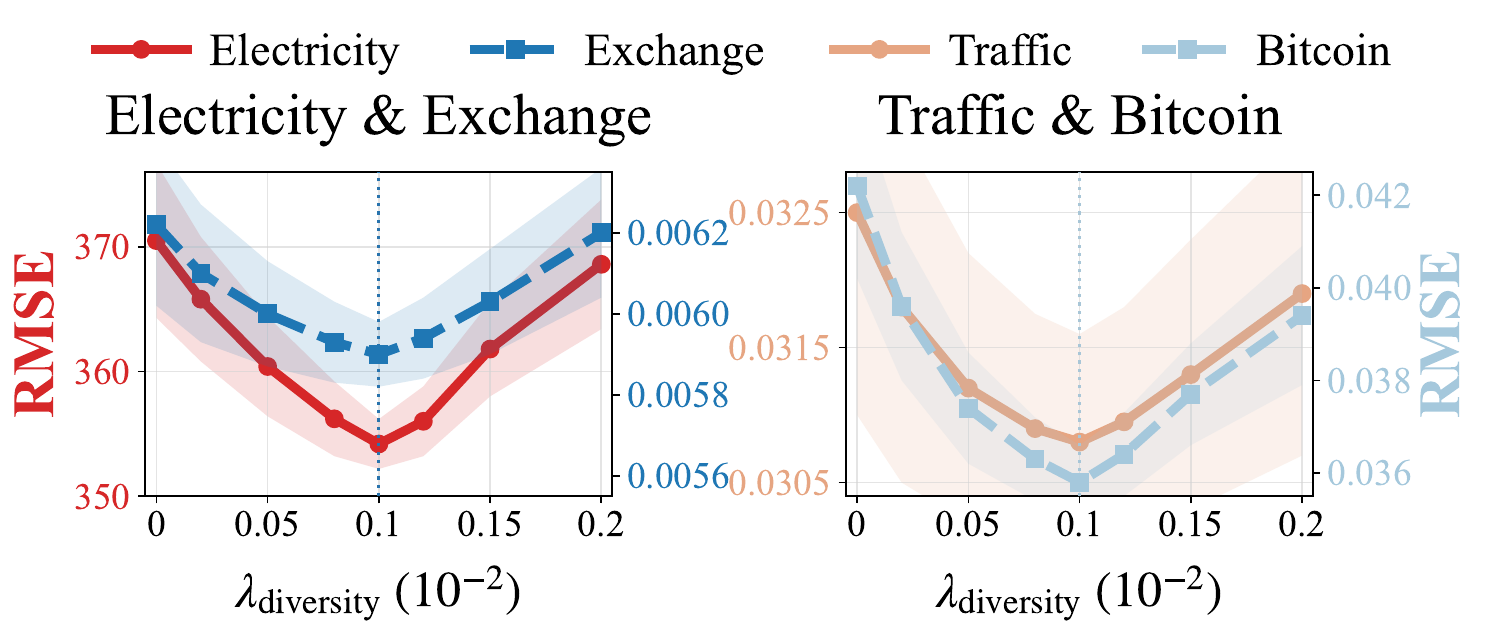}
    \caption{Diversity regularization.}
  \end{subfigure}

  \vspace{0.5em}  

  \begin{subfigure}{\linewidth}
    \centering
    \includegraphics[width=\linewidth]{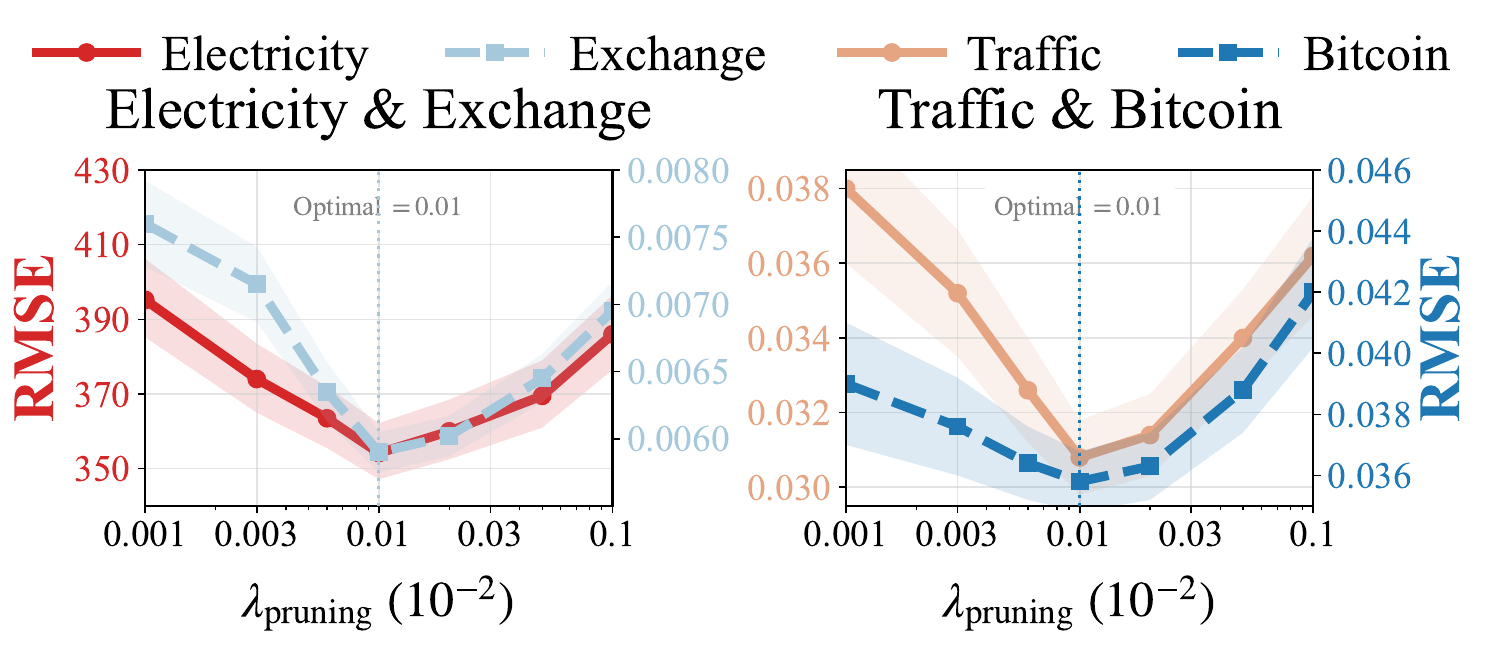}
    \caption{Pruning regularization.}
  \end{subfigure}

  \caption{Sensitivity analysis on diversity and pruning regularization. Both regularizers work best at moderate strength.}
  \label{fig:sensitivity_analysis}
\end{figure}

\subsection{Impact of Different Number of Initial Agents}
\label{sec:Different_Number_of_Initial_Agents}

We vary the initial population size from 2 to 10 on the Electricity dataset. Figure~\ref{fig:agent_population_impact} shows that larger populations generally obtain lower later-epoch MAPE, with $N=10$ giving the best final result. We therefore use 10 agents in the main experiments.

\begin{figure}[!ht]
  \centering
  \includegraphics[width=0.9\columnwidth]{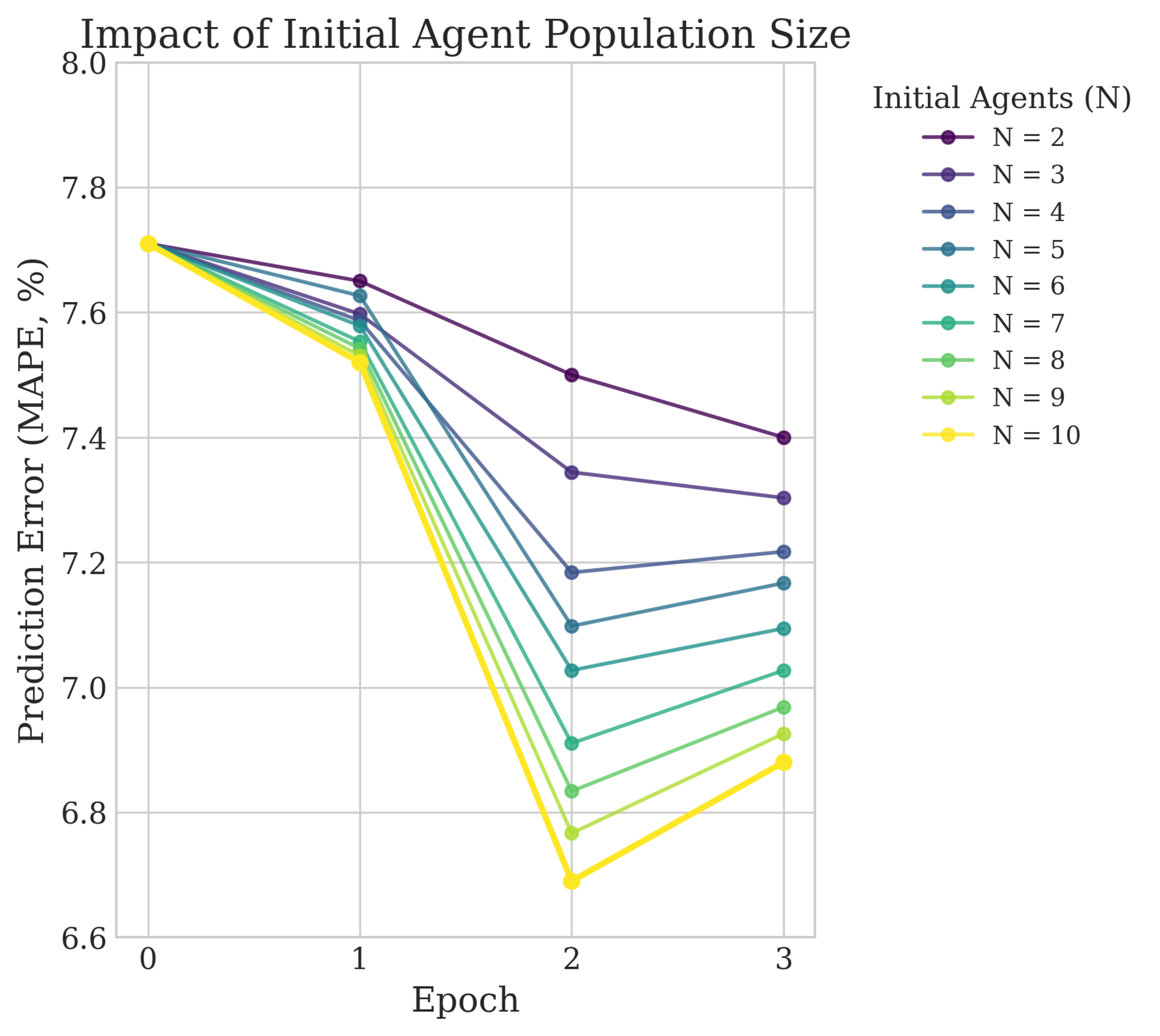}
\caption{Impact of the initial number of agents on prediction accuracy.}
  \label{fig:agent_population_impact}
\end{figure}

\subsection{Impact of Strategy Weight (\texorpdfstring{$\lambda_{PG}$}{lambda\_PG})}
\label{sec:impact_lambda_strategy}

The policy-update weight $\lambda_{PG}$ controls how strongly relative reward signals affect logic evolution. Table~\ref{tab:lambda_sensitivity} shows that a moderate value performs best: too little policy pressure reduces the method to a mostly fitness-weighted ensemble, while overly large weights destabilize forecasting.

\begin{table}[htbp]
    \centering
    \small
    \setlength{\tabcolsep}{3.5pt}
    \renewcommand{\arraystretch}{1.3}
    \begin{tabular}{ccccc}
        \toprule
        \multirow{2}{*}{$\lambda_{PG}$} & \multicolumn{4}{c}{Forecasting performance} \\
        \cmidrule(lr){2-5}
         & RMSE & MSE & MAE & MAPE \\
        \midrule
        0.0 (Ensemble) & 398.50 & 158.00 & 248.00 & 7.25\% \\
        0.1            & 375.20 & 142.00 & 232.00 & 6.85\% \\
        0.5 (CompEvo) & 354.18 & 125.44 & 220.15 & 6.45\% \\
        1.0            & 382.40 & 148.00 & 238.00 & 7.05\% \\
        5.0            & 435.00 & 189.00 & 268.00 & 7.95\% \\
        \bottomrule
    \end{tabular}
    \caption{Sensitivity analysis of $\lambda_{PG}$ on the Electricity dataset.}
    \label{tab:lambda_sensitivity}
\end{table}

\subsection{Impact of LLM Generation Temperature (\texorpdfstring{$T_{gen}$}{Tgen})}
\label{sec:Impact_of_Temperature}

We vary the generation temperature used by the LLM when producing textual logic. Figure~\ref{fig:temperature} shows that moderate stochasticity is sufficient for exploration without causing large performance variance. We set $T_{gen}=0.7$ in the main experiments.

\begin{figure}[!ht]
  \includegraphics[width=\columnwidth]{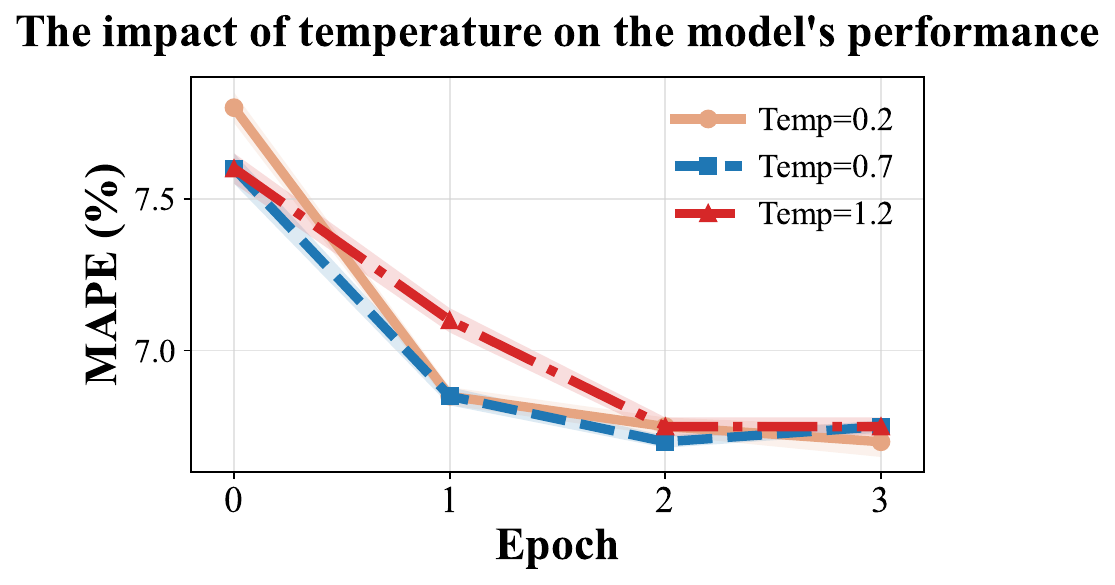}
  \caption{Impact of generation temperature on forecasting performance.}
  \label{fig:temperature}
\end{figure}

\subsection{Aggregation Temperature (\texorpdfstring{$\tau$}{tau})}
\label{sec:sensitivity_temperature}

This temperature controls the softmax aggregation weights in Equation~\ref{eq:weight_softmax}, and is distinct from the LLM generation temperature above. Table~\ref{tab:tau_impact_data} and Figure~\ref{fig:tau_sensitivity} show a U-shaped pattern: very small $\tau$ over-concentrates weight on a few agents, while very large $\tau$ weakens selection pressure. We use $\tau=0.5$.

\begin{table}[htbp]
    \centering
    \small
    \setlength{\tabcolsep}{3.5pt}
    \renewcommand{\arraystretch}{1.3}
    \begin{tabular}{ccccc}
        \toprule
        \multirow{2}{*}{Temperature ($\tau$)} & \multicolumn{4}{c}{Forecasting performance} \\
        \cmidrule(lr){2-5}
         & RMSE & MSE & MAE & MAPE \\
        \midrule
        0.1 & 378.50 & 144.00 & 235.00 & 6.92\% \\
        0.5 (CompEvo) & 354.18 & 125.44 & 220.15 & 6.45\% \\
        1.0 & 372.00 & 139.00 & 230.00 & 6.80\% \\
        5.0 & 415.00 & 170.00 & 255.00 & 7.60\% \\
        \bottomrule
    \end{tabular}
    \caption{Sensitivity analysis of aggregation temperature ($\tau$) on the Electricity dataset.}
    \label{tab:tau_impact_data}
\end{table}

\begin{figure}[htbp]
    \centering
    \includegraphics[width=0.9\columnwidth]{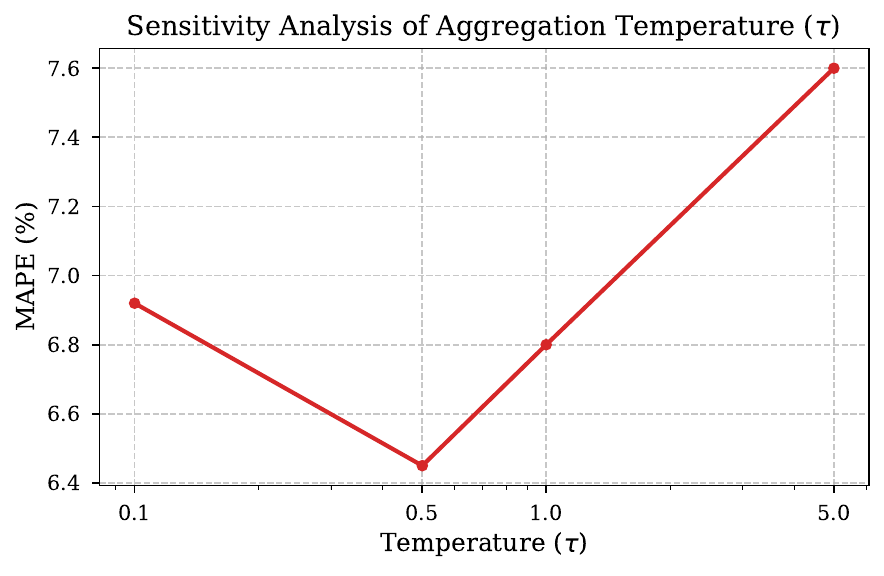}
    \caption{Aggregation temperature sensitivity.}
    \label{fig:tau_sensitivity}
\end{figure}

\subsection{Role of EMA Smoothing Factor (\texorpdfstring{$\beta$}{beta})}
\label{sec:impact_beta}

The EMA smoothing factor $\beta$ controls how quickly long-term fitness responds to recent forecasting rewards:
\begin{equation}
M_i^{(e+1)} = \beta M_i^{(e)} + (1 - \beta) R_i^{(e)}
\end{equation}
A small $\beta$ reacts quickly but is sensitive to noise, while a large $\beta$ gives stable but slow-moving fitness estimates. We use a moderate value to balance responsiveness and stability.

\subsection{Agent-side Instruction Robustness}
\label{sec:prompt_robustness}

We test prompt sensitivity by rephrasing only the agent-side logic-update instruction, while keeping the LLM judge prompt, model, data, and hyperparameters fixed. Figure~\ref{fig:prompt_settings} shows nearly overlapping MAPE curves before and after paraphrasing, suggesting that the result is not driven by a brittle wording choice in the agent-side update instruction.

\begin{figure}[t]
  \centering
  \includegraphics[width=0.8\columnwidth]{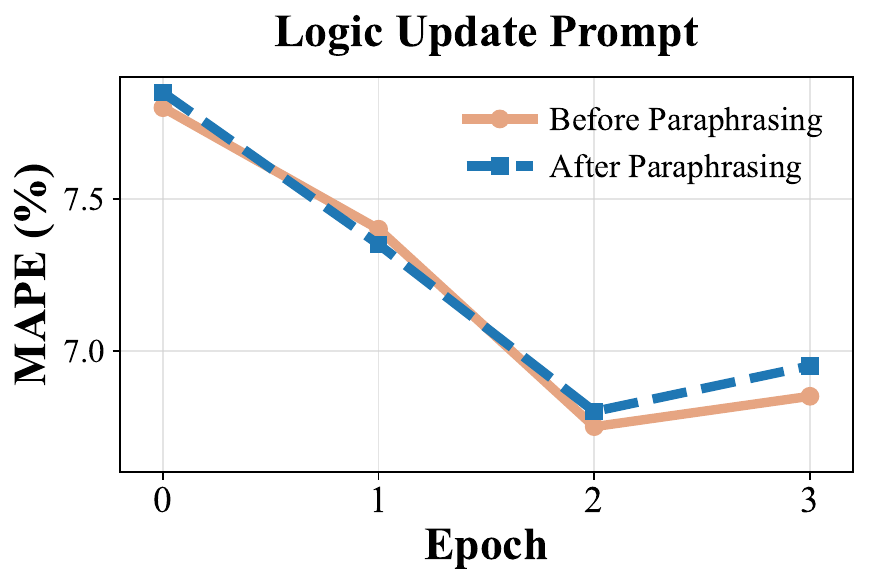}
  \caption{Robustness to paraphrasing the agent-side logic-update instruction. ``Before Paraphrasing'' uses the original agent-side update instruction, while ``After Paraphrasing'' uses a semantically equivalent rephrased version. The LLM judge prompt is kept fixed. Results are reported as MAPE across training epochs.}
  \label{fig:prompt_settings}
\end{figure}

\section{LUD Definition and LLM Judge Protocol}
\label{app:rius_prompt}

\subsection{Overview}
This appendix provides the complete definitions and computation procedure of our rubric-based Logic Update Degree (LUD), including (i) embedding-based update magnitude, (ii) the rubric-based Innovation Update Score (RIUS) with two indicators, and (iii) the exact LLM judge prompt template used to produce the rubric indicators. The rubric-based design follows recent work on expert-guided innovation measurement and structured LLM-based creativity evaluation~\cite{nowak2025aibasedmeasurementinnovationmapping,Zhao_2025,alrabeyah2024aut_agreement}.

\subsection{Formal Definition of LUD}
\label{app:lud_def}
For agent $i$ at evolutionary stage $t$, we denote its generated logic text as $logic_i^t$.
We first compute the embedding-based update magnitude between two consecutive stages:
\[
\Delta_i^t = 1 - \mathrm{sim}\!\left( \phi(logic_i^t), \phi(logic_i^{t-1}) \right),
\]
where $\phi(\cdot)$ is the bge-m3 encoder~\cite{chen2025m3embeddingmultilingualitymultifunctionalitymultigranularity} and $\mathrm{sim}(\cdot,\cdot)$ denotes cosine similarity.

Since embedding distance alone may capture superficial reformulations, we introduce a rubric-based Innovation Update Score (RIUS).
Given the pair $(logic_i^{t-1}, logic_i^t)$, an LLM judge returns two binary indicators:
\[
\begin{aligned}
s_{i,1}^t &\in \{0,1\}\ \text{(Strategy Shift)},\\
s_{i,2}^t &\in \{0,1\}\ \text{(Evidence Recombination)}.
\end{aligned}
\]
We define
\[
\mathrm{RIUS}_i^t = s_{i,1}^t + s_{i,2}^t,
\]
and compute the final Logic Update Degree as a gated embedding update:
\[
\mathrm{LUD}_i^t = \Delta_i^t \cdot \mathbb{I}(\mathrm{RIUS}_i^t \ge 1),
\]
where $\mathbb{I}(\cdot)$ is the indicator function.
Thus, a logic update contributes to LUD only if it exhibits at least one innovation-related rubric indicator.

\subsection{Rubric Indicators: Detailed Criteria}
\label{app:rius_rubric}
Both indicators are designed for pairwise comparison between $logic_i^{t-1}$ and $logic_i^t$ from the same agent.
The judge is instructed to focus on structural changes rather than surface-level rewriting.

\paragraph{Indicator 1: Strategy Shift ($s_{i,1}^t$).}
We set $s_{i,1}^t=1$ if and only if the current logic $logic_i^t$ shows a substantive change in the reasoning or news-filtering strategy relative to $logic_i^{t-1}$.
This includes, but is not limited to:
(i) switching the primary forecasting driver or reasoning lens (e.g., from macro-demand/trend-following to policy/regulation-driven reasoning);
(ii) changing the decision rule for selecting news signals (e.g., from broad coverage to event-triggered selection);
(iii) altering what evidence is prioritized in a way that changes the selection/interpretation behavior.
Otherwise, set $s_{i,1}^t=0$.

\paragraph{Indicator 2: Evidence Recombination ($s_{i,2}^t$).}
We set $s_{i,2}^t=1$ if and only if the current logic $logic_i^t$ introduces new causal cues or reorganizes multi-source evidence into a materially different explanatory structure relative to $logic_i^{t-1}$.
This includes, but is not limited to:
(i) combining multiple news sources into a new causal pathway (e.g., policy change $\rightarrow$ supply constraint $\rightarrow$ price shock);
(ii) integrating previously separate signals into a coherent chain of reasoning;
(iii) introducing a new causal factor that changes the explanatory mechanism.
Otherwise, set $s_{i,2}^t=0$.

\paragraph{Exclusion Rules (to avoid superficial variation).}
For both indicators, the judge must output $0$ when the difference is mainly:
(i) paraphrasing or minor rewording;
(ii) stylistic changes (formatting, fluency, verbosity);
(iii) adding generic statements without changing strategy or evidence structure;
(iv) repeating the same reasoning with different phrasing.
These exclusions ensure that rubric indicators capture innovation-related structural changes rather than surface-level textual variation.

\subsection{Validation of LUD}
\label{app:human_validation}
To validate the reasonableness of the LUD gating signal, we conduct a small-scale human audit on the LLM judge outputs.
We randomly sample logic-update pairs $(logic_i^{t-1}, logic_i^t)$ across datasets and methods, and ask three annotators to label the two rubric indicators (Strategy Shift and Evidence Recombination) following the same definitions in Appendix~\ref{app:rius_rubric}.
We report inter-annotator agreement and the agreement between the LLM judge and the majority human label.
Overall, the judge decisions are consistent with human consensus, following established LLM-as-a-judge validation paradigms in text evaluation \cite{zheng2023judging, chan2024chateval}. This suggests that LUD mainly reflects innovation-related structural updates rather than superficial paraphrasing.

\paragraph{Robustness of LUD across LLM judges.}
\label{sec:judge_robustness}
To test whether the LUD signal is tied to a particular evaluator, we re-score 500 randomly sampled logic-update pairs using two additional LLM judges, GPT-5.2 and Claude 4.6 Opus, while keeping the rubric and prompt template fixed. Table~\ref{tab:judge_robustness} shows substantial agreement across LLM judges in both categorical decisions and continuous LUD scores. This suggests that the innovation signal captured by LUD is not an artifact of one specific LLM judge.

\begin{table}[h]
\centering
\small
\setlength{\tabcolsep}{4pt}
\renewcommand{\arraystretch}{1.05}
\resizebox{\columnwidth}{!}{%
\begin{tabular}{lccc}
\toprule
Judge Pair & Cohen's $\kappa$ & Pearson $r$ & Agreement \\
\midrule
Primary judge vs.\ Claude 4.6 & 0.75 & 0.81 & Substantial \\
Primary judge vs.\ GPT-5.2 & 0.72 & 0.76 & Substantial \\
GPT-5.2 vs.\ Claude 4.6 & 0.68 & 0.74 & Moderate--Substantial \\
\bottomrule
\end{tabular}
}
\caption{Robustness of LUD across LLM judges.}
\label{tab:judge_robustness}
\end{table}

\paragraph{Correlation between LUD and forecast improvement.}
\label{sec:lud_error_correlation}
We further examine whether larger logic updates are associated with stronger forecasting gains. On the Electricity dataset, we collect 1,200 update steps and correlate each step's LUD with the corresponding error reduction in RMSE. We observe a significant positive relationship ($r=0.49$, $p<0.001$), indicating that larger, rubric-validated logic changes are more likely to coincide with meaningful forecasting improvement.

\begin{table}[h]
\centering
\small
\setlength{\tabcolsep}{5pt}
\renewcommand{\arraystretch}{1.05}
\resizebox{\columnwidth}{!}{%
\begin{tabular}{lccc}
\toprule
LUD Regime & Avg. LUD & Avg. $\Delta$RMSE Gain & Interpretation \\
\midrule
Low & 0.09 & 1.12 & Statistically weak local adjustment \\
Medium & 0.44 & 7.85 & Steady incremental optimization \\
High & 0.76 & 14.20 & High-impact strategic revision \\
\bottomrule
\end{tabular}
}
\caption{Relationship between LUD magnitude and forecasting improvement on the Electricity dataset.}
\label{tab:lud_error_correlation}
\end{table}

\subsection{LLM judge settings}
\label{app:judge_setting}
We use a deterministic LLM judge configuration to improve reproducibility:
temperature $=0$ and a fixed prompt template.
The judge is blind to the method identity (e.g., CompEvo vs.\ baselines) and receives only the logic pair.

\subsection{LLM Judge Prompt Template}
\label{sec:llm_judge_prompt}

To compute $(s_{i,1}^t, s_{i,2}^t)$, we use a deterministic LLM judge setting
with temperature $=0$ and a fixed prompt template, as shown in Figure~\ref{fig:prompt_templates}(a).
The judge receives two consecutive logics from the same agent and outputs a JSON object with two binary indicators.

\section{Instruction Templates}
\label{sec:prompt_design}

We provide the agent-side logic-update instruction template for reproducibility. As shown in Figure~\ref{fig:prompt_templates}(b), the template exposes the agent's current logic, recent reward, and peer context, and asks the agent to synthesize a revised logic rather than copy another agent.

\begin{figure*}[p]
\centering
\scriptsize

\begin{subfigure}{0.98\textwidth}
\centering
\begin{tcolorbox}[
    title={LLM Judge Prompt for Logic Update Evaluation},
    enhanced,
    colback=white,
    colframe=black!35,
    colbacktitle=black!55,
    coltitle=white,
    boxrule=0.45pt,
    arc=1pt,
    left=4pt,
    right=4pt,
    top=3pt,
    bottom=3pt,
    fontupper=\footnotesize,
    width=\linewidth,
    before skip=2pt,
    after skip=2pt
]

\tcbsubtitle[
    colback=black!8,
    colframe=black!20,
    coltitle=black,
    fonttitle=\footnotesize\bfseries,
    left=3pt,
    right=3pt,
    top=1pt,
    bottom=1pt
]{System Prompt}

You are a strict evaluator of logic updates for time-series forecasting with news.
Your job is not to judge whether the logic is good, but whether the update shows substantive structural change.
Follow the rubric and output only valid JSON.

\vspace{0.15em}

\tcbsubtitle[
    colback=black!8,
    colframe=black!20,
    coltitle=black,
    fonttitle=\footnotesize\bfseries,
    left=3pt,
    right=3pt,
    top=1pt,
    bottom=1pt
]{User Prompt}

Given two consecutive logics produced by the same agent:

Previous logic ($t-1$): \{logic\_prev\}

Current logic ($t$): \{logic\_curr\}

Decide two binary indicators:

1) \textbf{Strategy Shift (0/1):} output 1 iff the current logic changes the reasoning or news-filtering strategy compared to the previous logic, e.g., changes the main driver, decision rule, or prioritization strategy. Output 0 for paraphrasing, minor rewording, style changes, or generic additions.

2) \textbf{Evidence Recombination (0/1):} output 1 iff the current logic introduces new causal cues or reorganizes multi-source evidence into a materially different explanatory structure. Output 0 for paraphrasing or shallow elaboration.

Return only a valid JSON object. The values must be either 0 or 1.

\begin{lstlisting}[basicstyle=\ttfamily\footnotesize]
{
  "strategy_shift": 0,
  "evidence_recombination": 1
}
\end{lstlisting}

\end{tcolorbox}
\caption{LLM judge prompt for evaluating logic updates.}
\label{fig:logic_update_judge_prompt}
\end{subfigure}

\vspace{0.25em}

\begin{subfigure}{0.98\textwidth}
\centering
\begin{tcolorbox}[
    title={Agent-side Logic-update Instruction},
    enhanced,
    colback=white,
    colframe=black!35,
    colbacktitle=black!55,
    coltitle=white,
    boxrule=0.45pt,
    arc=1pt,
    left=4pt,
    right=4pt,
    top=3pt,
    bottom=3pt,
    fontupper=\footnotesize,
    width=\linewidth,
    before skip=2pt,
    after skip=2pt
]

\tcbsubtitle[
    colback=black!8,
    colframe=black!20,
    coltitle=black,
    fonttitle=\footnotesize\bfseries,
    left=3pt,
    right=3pt,
    top=1pt,
    bottom=1pt
]{System Instruction}

You are a strategy analyst for a time-series forecasting agent. Your goal is to refine the agent's evidence-seeking strategy, realized as textual logic, to improve its prediction accuracy.

\vspace{0.15em}

\tcbsubtitle[
    colback=black!8,
    colframe=black!20,
    coltitle=black,
    fonttitle=\footnotesize\bfseries,
    left=3pt,
    right=3pt,
    top=1pt,
    bottom=1pt
]{Context Input}

Agent ID: Agent\_3.

Agent's Current Logic: ``Prioritize news about extreme weather conditions like heatwaves or storms as they directly impact residential cooling/heating demand.''

Agent's Recent Performance: $-0.0450$ (Negative MSE, higher is better).

\vspace{0.15em}

\tcbsubtitle[
    colback=black!8,
    colframe=black!20,
    coltitle=black,
    fonttitle=\footnotesize\bfseries,
    left=3pt,
    right=3pt,
    top=1pt,
    bottom=1pt
]{Peer Information}

Agent\_5 (Reward: $-0.0320$): ``Focus on government stimulus packages and industrial restart plans that drive commercial electricity load.''

Agent\_2 (Reward: $-0.0480$): ``Track social media sentiment regarding holiday travel plans.''

\vspace{0.15em}

\tcbsubtitle[
    colback=black!8,
    colframe=black!20,
    coltitle=black,
    fonttitle=\footnotesize\bfseries,
    left=3pt,
    right=3pt,
    top=1pt,
    bottom=1pt
]{Task}

Based on all this information, generate a new, refined logic for the agent. The new logic should be a single, concise sentence. It should be an evolution of the current logic, either by specializing, generalizing, or correcting it based on the performance feedback. Do not just copy a peer's logic; synthesize a new one.

\vspace{0.15em}

\tcbsubtitle[
    colback=black!8,
    colframe=black!20,
    coltitle=black,
    fonttitle=\footnotesize\bfseries,
    left=3pt,
    right=3pt,
    top=1pt,
    bottom=1pt
]{Output Format}

New Refined Logic: [Your logic here]

\vspace{0.15em}

\tcbsubtitle[
    colback=black!8,
    colframe=black!20,
    coltitle=black,
    fonttitle=\footnotesize\bfseries,
    left=3pt,
    right=3pt,
    top=1pt,
    bottom=1pt
]{Example Model Output}

New Refined Logic: Combine monitoring of extreme weather events with news on industrial recovery policies to capture both the volatility of residential usage and the baseline shifts in commercial demand.

\end{tcolorbox}
\caption{Agent-side logic-update instruction template.}
\label{fig:agent_side_logic_update_prompt}
\end{subfigure}

\caption{Prompt templates used in CompEvo.}
\label{fig:prompt_templates}
\end{figure*}

\fi

\end{document}